\documentclass[11pt]{article}
\usepackage[numbers]{natbib}
\usepackage[colorlinks=true, citecolor=blue]{hyperref}
\usepackage{amsmath,amssymb,amsthm}
\usepackage{algorithm,algorithmic}
\usepackage[usenames,dvipsnames]{xcolor}
\usepackage{fullpage}
\usepackage{MnSymbol}
\usepackage{bbm}

\newtheorem{theorem}{Theorem}
\newtheorem{lemma}[theorem]{Lemma}

\newtheorem{example}{Example}

\theoremstyle{definition}
\newtheorem{definition}{Definition}

\DeclareMathOperator*{\Ex}{\vphantom{p}\mathbb{E}}
\let\inf\undef
\DeclareMathOperator*{\inf}{\vphantom{p}inf}
\let\sup\undef
\DeclareMathOperator*{\sup}{\vphantom{p}sup}

\begin{document}

\setlength{\parskip}{2mm}
\setlength{\parindent}{0pt}

\title{Competing With Strategies}

\author{Wei Han \\ Univ. of Pennsylvania \and Alexander Rakhlin\\ Univ. of Pennsylvania  \and Karthik Sridharan\\ Univ. of Pennsylvania }	

%%%%%%%%%% NOTATION %%%%%%%%%%
% Shortcuts
\newcommand{\mbb}[1]{\mathbb{#1}}
\newcommand{\mbf}[1]{\mathbf{#1}}
\newcommand{\mc}[1]{\mathcal{#1}}
\newcommand{\mrm}[1]{\mathrm{#1}}
\newcommand{\trm}[1]{\textrm{#1}}

% General math
%\newcommand{\v}{\mbb{v}}
\newcommand{\norm}[1]{\left\|#1\right\|}
\newcommand{\sign}{\mrm{sign}}
\newcommand{\argmin}[1]{\underset{#1}{\mrm{argmin}} \ }
\newcommand{\argmax}[1]{\underset{#1}{\mrm{argmax}} \ }
\newcommand{\reals}{\mathbb{R}}
\newcommand{\E}[1]{\mathbb{E}\left[ #1 \right]} % brackets
\newcommand{\Ebr}[1]{\mathbb{E}\left\{ #1 \right\}} % curly braces
\newcommand{\En}{\mathbb{E}}  % no brackets or braces
\newcommand{\Eu}[1]{\underset{#1}{\mathbb{E}}}  % subscript no braces
\newcommand{\Ebar}{\Hat{\Hat{\mathbb{E}}}}  % no brackets or braces
\newcommand{\Esbar}[2]{\Hat{\Hat{\mathbb{E}}}_{#1}\left[ #2 \right]} % subscript + brackets
\newcommand{\Es}[2]{\mathbb{E}_{#1}\left[ #2 \right]} % subscript + brackets
\newcommand{\Ps}[2]{\mathbb{P}_{#1}\left[ #2 \right]}
\newcommand{\Prob}{\mathbb{P}}
\newcommand{\conv}{\operatorname{conv}}
\newcommand{\inner}[1]{\left\langle #1 \right\rangle}
\newcommand{\ip}[2]{\left<#1,#2\right>}
\newcommand{\lv}{\left\|}
\newcommand{\rv}{\right\|}
\newcommand{\Phifunc}[1]{\Phi\left(#1\right)}
\newcommand{\ind}[1]{{\bf 1}\left\{#1\right\}}
\newcommand{\tr}{\ensuremath{{\scriptscriptstyle\mathsf{T}}}}
\newcommand{\eqdist}{\stackrel{\text{d}}{=}}
\newcommand{\alphT}{\widehat{\alpha}(T)}
\newcommand{\PD}{\mathcal P}
\newcommand{\QD}{\mathcal Q}
\newcommand{\jp}{\ensuremath{\mathbf{p}}}
\newcommand{\rh}{\boldsymbol{\rho}}
\newcommand{\proj}{\text{Proj}}
\newcommand{\Eunderone}[1]{\underset{#1}{\En}}
\newcommand{\Eunder}[2]{\underset{\underset{#1}{#2}}{\En}}
\newcommand{\bphi}{\boldsymbol\phi}
% Bold symbols
\newcommand\s{\mathbf{s}}
\newcommand\w{\mathbf{w}}
\newcommand\x{\mathbf{x}}
\newcommand\y{\mathbf{y}}
\newcommand\z{\mathbf{z}}
\newcommand\f{\mathbf{f}}

\renewcommand\v{\mathbf{v}}

% Calligraphic symbols
\newcommand\cB{\mathcal{B}}
\newcommand\cC{\mathcal{C}}
\newcommand\cD{\mathcal{D}}
\newcommand\cL{\mathcal{L}}
\newcommand\cN{\mathcal{N}}
\newcommand\X{\mathcal{X}}
\newcommand\Y{\mathcal{Y}}
\newcommand\Z{\mathcal{Z}}
\newcommand\F{\mathcal{F}}
\newcommand\G{\mathcal{G}}
\newcommand\cH{\mathcal{H}}
\newcommand\N{\mathcal{N}}
\newcommand\M{\mathcal{M}}
\newcommand\cO{\mathcal{O}}
\newcommand\W{\mathcal{W}}
\newcommand\Nhat{\mathcal{\widehat{N}}}
\newcommand\Diff{\mathcal{G}}
\newcommand\Compare{\boldsymbol{B}}
\newcommand\RH{\eta} % Radius of \cH
\newcommand\metricent{\N_{\mathrm{metric}}} % metric entropy

\newcommand{\karthik}[1]{{\color{red} Karthik: #1}}
\newcommand{\sasha}[1]{{\color{blue} Sasha: #1}}
\newcommand{\ohad}[1]{{\color{green} Ohad: #1}}

% Notations
\newcommand\ldim{\mathrm{Ldim}}
\newcommand\fat{\mathrm{fat}}
\newcommand\Img{\mbox{Img}}
\newcommand\sparam{\sigma} % Strong convexity constant
\newcommand\Psimax{\ensuremath{\Psi_{\mathrm{max}}}}

% Calligraphic notations
\newcommand\Rad{\mathfrak{R}}
\newcommand\Val{\mathcal{V}}
\newcommand\Valdet{\mathcal{V}^{\mathrm{det}}}
\newcommand\Dudley{\mathfrak{D}}
\newcommand\Reg{\mbf{Reg}}
\newcommand\D{\mbf{D}}
\renewcommand\P{\mbf{P}}

% For the examples section
\newcommand\Xcvx{\X_\mathrm{cvx}}
\newcommand\Xlin{\X_\mathrm{lin}}
\newcommand{\Relax}[3]{\mbf{Rel}_{#1}\left(#2 \middle| #3 \right)}
\newcommand{\Rel}[2]{\mbf{Rel}_{#1}\left(#2 \right)}

\newcommand{\multiminimax}[1]{\ensuremath{\left\llangle #1\right\rrangle}}

\def\deq{\triangleq}

\newcommand\loss{\bold{\ell}}
\newcommand\RadStat{\mathcal{R}^{iid}}
\newcommand\g{\mathbf{g}}
\newcommand\h{\mathbf{h}}
\newcommand\vv{\mathbf{v}}
\newcommand{\card}[1]{\text{card}\left(#1\right)}
\newcommand\cZ{\mathcal{Z}}
\newcommand\cX{\mathcal{X}}

\newcommand{\bX}{{\mathbf X}}
\newcommand{\bY}{{\mathbf Y}}

%%%%%%%%%%%%%%%%%%%%%%%%%%%%%%

\maketitle

\begin{abstract}
	We study the problem of online learning with a notion of regret defined with respect to a set of strategies. We develop tools for analyzing the minimax rates and for deriving regret-minimization algorithms in this scenario. While the standard methods for minimizing the usual notion of regret fail, through our analysis we demonstrate existence of regret-minimization methods that compete with such sets of strategies as: autoregressive algorithms, strategies based on statistical models, regularized least squares, and follow the regularized leader strategies. In several cases we also derive efficient learning algorithms. 
\end{abstract}

\section{Introduction} 

The common criterion for evaluating an online learning algorithm is \emph{regret}, that is the difference between the cumulative loss of the algorithm and the cumulative loss of the best fixed decision, chosen in hindsight. While much work has been done on understanding no-regret algorithms, such a definition of regret against a fixed decision often draws criticism: even if regret is small, the cumulative loss of a best \emph{fixed} action can be large, thus rendering the result uninteresting. To address this problem, various generalizations of the regret notion have been proposed, including regret with respect to the cost of a ``slowly changing'' compound decision. While being a step in the right direction, such definitions are still ``static'' in the sense that the decision of each compound comparator per step does not depend on the sequence of realized outcomes.

Arguably, a more interesting (and more difficult to deal with) notion is that of performing as well as a set of \emph{strategies} (or, \emph{algorithms}). A strategy $\pi$ is a sequence of functions $\pi_t$, for each time period $t$, mapping the observed outcomes to the next action. Of course, if the collection of such strategies is finite, we may disregard their dependence on the actual sequence and treat each strategy as a black box expert. This is precisely the reason the Multiplicative Weights and other expert algorithms gained such popularity. However, this ``black box'' approach is not always desirable since some measure of the  ``effective number'' of experts must play a role in the complexity of the problem: experts that predict similarly should not count as two independent ones. But what is a notion of closeness of two strategies? Imagine that we would like to develop an algorithm that incurs loss comparable to that of the best of an infinite family of strategies. To obtain such a statement, one may try to discretize the space of strategies and invoke the black-box experts method. As we show in this paper, such an approach will not always work. Instead, we present a theoretical framework for the analysis of ``competing against strategies'' and for algorithmic development, based on the ideas in \citep{RakSriTew10nips, RakShaSri12nips}. 

The strategies considered in this paper are termed ``simulatable experts'' in \citep{PLG}. The authors also distinguish {\em static} and {\em non-static} experts. In particular, for static experts and absolute loss, \cite{cesa1999prediction} were able to show that problem complexity is governed by the geometry of the class of static experts as captured by its i.i.d. Rademacher averages. For nonstatic experts, however, the authors note that ``unfortunately we do not have a characterization of the minimax regret by an empirical process'', due to the fact that the sequential nature of the online problems is at odds with the i.i.d.-based notions of classical empirical process theory. In recent years, however, a martingale generalization of empirical process theory has emerged, and these tools were shown to characterize learnability of online supervised learning, online convex optimization, and other scenarios \citep{RakSriTew10nips, BenPalSha09}. Yet, the machinery developed so far is not directly applicable to the case of general simulatable experts which can be viewed as mappings from an ever-growing set of histories to the space of actions. The goal of this paper is precisely this: to extend the non-constructive as well as constructive techniques of \citep{RakSriTew10nips, RakShaSri12nips} to simulatable experts. We analyze a number of examples with the developed techniques, but we must admit that our work only scratches the surface. We can imagine further research developing methods that compete with interesting gradient descent methods (parametrized by step size choices), with Bayesian procedures (parametrized by choices of priors), and so on. We also note the connection to online algorithms, where one typically aims to prove a bound on the competitive ratio. Our results can be seen in that light as implying a competitive ratio of one.

We close the introduction with a high-level outlook, which builds on the ideas of \cite{MerFed98}. Imagine we are faced with a sequence of data from a probabilistic source, such as a $k$-Markov model with unknown transition probabilities. A well developed statistical theory  tells us how to estimate the parameter \emph{under the assumption that the model is correct}. We may view an estimator as a \emph{strategy} for predicting the next outcome. Suppose we have a set of possible models, with a good prediction strategy for each model. Now, let us lift the assumption that the sequence is generated by one of these models, and set the goal as that of performing as well as the best prediction strategy. In this case, if the observed sequence is indeed given by one of the models, our loss will be small because one of the strategies will perform well. If not, we still have a valid statement that does not rely on the fact that the model is ``well specified''. To illustrate the point, we will exhibit an example where we can compete with the set of all Bayesian strategies (parametrized by priors). We then obtain a statement that we perform as well as the best of them without assuming that the model is correct.

The paper is organized as follows. In Section~\ref{sec:minimax}, we extend the minimax analysis of online learning problems to the case of competing with a set of strategies. In Section~\ref{sec:autoregressive}, we show that it is possible to compete with a set of autoregressive strategies, and that the usual online linear optimization algorithms do not attain the optimal bounds. We then derive an optimal and computationally efficient algorithm for one of the proposed regimes. In Section~\ref{sec:statmodels} we describe the general idea of competing with statistical models that use sufficient statistics, and demonstrate an example of competing with a set of strategies parametrized by priors. For this example, we derive an optimal and efficient randomized algorithm. In Section~\ref{sec:rls}, we turn to the question of competing with regularized least squares algorithms indexed by the choice of a shift and a regularization parameter. In Section~\ref{sec:ftrl}, we consider online linear optimization and show that it is possible to compete with Follow the Regularized Leader methods parametrized by a shift and by a step size schedule.

\section{Minimax Regret and Sequential Rademacher Complexity}
\label{sec:minimax}

We consider the problem of online learning, or sequential prediction, that consists of $T$ rounds. At each time $t=\{1,\ldots,T\} \deq [T]$, the learner makes a prediction $f_t\in\F$ and observes an outcome $z_t\in\cZ$, where $\F$ and $\cZ$ are abstract sets of decisions and outcomes. Let us fix a loss function $\ell:\F\times\Z\mapsto \reals$ that measures the quality of prediction. A \emph{strategy} $\pi = (\pi_t)_{t=1}^T$ is a sequence of functions  $\pi_t:\Z^{t-1}\mapsto\F$ mapping history of outcomes to a decision. Let $\Pi$ denote a set of strategies. The regret with respect to $\Pi$ is the difference between the cumulative loss of the player and the cumulative loss of the best strategy
$$\Reg_T = \sum_{t=1}^T \ell(f_t,z_t)-\inf_{\pi \in \Pi} \sum_{t=1}^T \ell(\pi_t(z_{1:t-1}),z_t).$$
where we use the notation $z_{1:k}\deq \{z_1,\ldots,z_k\}$. We now define the value of the game against a set $\Pi$ of strategies as
$$\mathcal{V}_T(\Pi) \deq \inf_{q_1 \in \QD} \sup_{z_1 \in \cZ} \Ex_{f_1 \sim q_1} \dots \inf_{q_T \in \QD} \sup_{z_T \in \cZ} \Ex_{f_T \sim q_T} \left[  \Reg_T \right]$$
where  $\QD$ and $\PD$ are the sets of probability distributions on $\F$ and $\cZ$, correspondingly. It was shown in \citep{RakSriTew10nips} that one can derive non-constructive upper bounds on the value through a process of sequential symmetrization, and in \citep{RakShaSri12nips} it was shown that these non-constructive bounds can be used as relaxations to derive an algorithm. This is the path we take in this paper.

Let us describe an important variant of the above problem -- that of \emph{supervised learning}. Here, before making a real-valued prediction $\hat{y}_t$ on round $t$, the learner observes side information $x_t\in \X$. Simultaneously, the actual outcome $y_t\in\Y$ is chosen by Nature. A strategy can therefore depend on the history $x_{1:t-1},y_{t-1}$ \emph{and} the current $x_t$, and we write such strategies as $\pi_t(x_{1:t},y_{1:t-1})$, with $\pi_t:\X^t\times \Y^{t-1} \mapsto \Y$. Fix some loss function $\ell(\hat{y},y)$. The value $\mathcal{V}^{S}_T(\Pi)$ is then defined as
$$\sup_{x_1}\inf_{q_1 \in \Delta(\Y)} \sup_{y_1 \in \Y} \Ex_{\hat{y}_1 \sim q_1}  \ldots \sup_{x_T}\inf_{q_T \in \Delta(\Y)} \sup_{y_T \in \Y} \Ex_{\hat{y}_T \sim q_T}  \left[ \sum_{t=1}^T \ell(\hat{y}_t,y_t)-\inf_{\pi \in \Pi} \sum_{t=1}^T \ell(\pi_t(x_{1:t}, y_{1:t-1}),y_t)\right]$$

To proceed, we need to define a notion of a tree. A $\Z$-valued tree $\z$ is a sequence of mappings $\{\z_1,\ldots,\z_T\}$ with $\z_t:\{\pm1\}^{t-1}\mapsto \Z$. Throughout the paper, $\epsilon_t \in \{\pm 1\}$ are i.i.d. Rademacher variables, and a realization of $\epsilon = (\epsilon_1, \dots, \epsilon_T)$ defines a \emph{path} on the tree, given by $\z_{1:t}(\epsilon) \deq (\z_1(\epsilon), \dots, \z_t(\epsilon))$ for any $t\in[T]$. We write $\z_t(\epsilon)$ for $\z_t(\epsilon_{1:t-1})$. By convention, a sum $\sum_{a}^b = 0$ for $a>b$ and for simplicity assume that no loss is suffered on the first round.

\begin{definition}
	Sequential Rademacher complexity of the set $\Pi$ of strategies is defined as
\begin{align}
	\label{eq:seq_rad}
	\Rad(\loss, \Pi) ~\deq \sup_{\w,\z}\mathbb{E}_{\epsilon} \sup_{\pi \in \Pi}\left[ \sum_{t=1}^T \epsilon_t\loss(\pi_t(\w_1(\epsilon),\ldots,\w_{t-1}(\epsilon)), \z_t(\epsilon)) \right]
\end{align}
where the supremum is over two $\mathcal{Z}$-valued trees $\z$ and $\w$ of depth $T$.
\end{definition}
The $\w$ tree can be thought of as providing ``history'' while $\z$ providing ``outcomes''. We shall use these names throughout the paper. The reader might notice that in the above definition, the outcomes and history are decoupled. We now state the main result:
\begin{theorem}
	\label{thm:compete_with_strategies}
	The value of prediction problem with a set $\Pi$ of strategies is upper bounded as
	\begin{align*}
		\Val_T(\Pi)&\leq 	2\Rad(\loss, \Pi)
	\end{align*}
\end{theorem}
While the statement is visually similar to those in  \cite{RakSriTew10nips,RakSriTew11nips}, it does not follow from these works. Indeed, the proof (which appears in Appendix) needs to deal with the  additional complications stemming from the dependence of strategies on the history. Further, we provide the proof for a more general case when sequences $z_1,\ldots,z_T$ are not arbitrary but need to satisfy constraints. 

As we show below, the sequential Rademacher complexity on the right-hand side allows us to analyze general non-static experts, thus addressing the question raised in \citep{cesa1999prediction}. As the first step, we can ``erase'' a Lipschitz loss function (see \cite{rakhlin:notes} for more details), leading to the sequential Rademacher complexity of $\Pi$ without the loss and without the $\z$ tree:
$$\Rad(\Pi) \deq \sup_{\w} \Rad(\Pi,\w) \deq \sup_{\w}\mathbb{E}_{\epsilon}\sup_{\pi \in \Pi} \left[ \sum_{t=1}^T \epsilon_t \pi_t(\w_{1:t-1}(\epsilon))  \right] $$
For example, suppose $\cZ=\{0,1\}$, the loss function is the indicator loss, and strategies have potentially dependence on the full history. Then one can verify that
\begin{align}
	\label{eq:contraction_indicator}
	&\sup_{\w,\z}\mathbb{E}_{\epsilon}\sup_{\pi \in \Pi^k} \left[ \sum_{t=1}^T \epsilon_t \ind{\pi_t(\w_{1:t-1}(\epsilon))\neq \z_t(\epsilon)} \right] \notag \\
	&=\sup_{\w,\z}\mathbb{E}_{\epsilon}\sup_{\pi \in \Pi^k} \left[ \sum_{t=1}^T \epsilon_t \big(\pi_t(\w_{1:t-1}(\epsilon))(1-2\z_t(\epsilon)\big) + \z_t(\epsilon)) \right] = \Rad(\Pi)
	%&= \sup_{\w}\mathbb{E}_{\epsilon}\sup_{\pi \in \Pi^k} \left[ \sum_{t=1}^T \epsilon_t \pi_t(\w_{1}(\epsilon),\ldots,\w_{t-1}(\epsilon))  \right]
\end{align}
The same result holds when $\F=[0,1]$ and $\loss$ is the absolute loss. The process of ``erasing the loss'' (or, contraction) extends quite nicely to problems of supervised learning. Let us state the second main result: \begin{theorem}
	\label{thm:main_lip_contraction}
	Suppose the loss function $\ell:\Y\times\Y\mapsto\reals$ is convex and $L$-Lipschitz in the first argument, and let $\Y=[-1,1]$. Then
	$$\mathcal{V}^{S}_T(\Pi) \leq 2L\sup_{\x,\y}\Ex_{\epsilon}\sup_{\pi\in\Pi} \left[ \sum_{t=1}^T \epsilon_t \pi_t(\x_{1:t}(\epsilon), \y_{1:t-1}(\epsilon))\right]$$
	where $(\x_{1:t}(\epsilon), \y_{1:t-1}(\epsilon))$ naturally takes place of $\w_{1:t-1}(\epsilon)$ in Theorem~\ref{thm:compete_with_strategies}. Further, if $\Y=[-1,1]$ and $\loss(\hat{y},y)=|\hat{y}-y|$, 
	~~~$\mathcal{V}^{S}_T(\Pi) \geq \sup_{\x}\Ex_{\epsilon}\sup_{\pi\in\Pi} \left[ \sum_{t=1}^T \epsilon_t \pi_t(\x_{1:t}(\epsilon), \epsilon_{1:t-1})\right]$.
\end{theorem}

Let us present a few simple examples as a warm-up.

\begin{example}[History-independent strategies]
Let $\pi^f\in\Pi$ be constant history-independent strategies $\pi^f_1=\ldots=\pi^f_T= f\in\F$. Then \eqref{eq:seq_rad} recovers the definition of sequential Rademacher complexity in \cite{RakSriTew10nips}.
\end{example}

\begin{example}[Static experts]
For static experts, each strategy $\pi$ is a predetermined sequence of outcomes, and we may therefore associate each $\pi$ with a vector in $\cZ^T$. A direct consequence of Theorem~\ref{thm:main_lip_contraction} for any convex $L$-Lipschitz loss is that
\begin{align*}
	\Val(\Pi) \leq  2L\mathbb{E}_{\epsilon} \sup_{\pi \in \Pi}\left[ \sum_{t=1}^T \epsilon_t \pi_t \right]
\end{align*}
which is simply the classical i.i.d. Rademacher averages. For the case of $\F=[0,1]$, $\Z=\{0,1\}$, and the absolute loss, this is the result of \cite{cesa1999prediction}.
\end{example}

\begin{example}[Finite-order Markov strategies]
Let $\Pi^k$ be a set of strategies that only depend on the $k$ most recent  outcomes to determine the next move. Theorem~\ref{thm:compete_with_strategies}  implies that the value of the game is upper bounded as
\begin{align*}
	\textstyle \Val(\Pi^k) \leq 2\sup_{\w,\z}\mathbb{E}_{\epsilon}\sup_{\pi \in \Pi^k} \left[ \sum_{t=1}^T \epsilon_t\loss(\pi_t(\w_{t-k}(\epsilon),\ldots,\w_{t-1}(\epsilon)), \z_t(\epsilon)) \right]
\end{align*}
Now, suppose that $\cZ$ is a finite set, of cardinality $s$. Then there are effectively $s^{s^k}$ strategies $\pi$. The bound on the sequential Rademacher complexity then scales as $\sqrt{2s^k\log(s) T}$, recovering the result of \cite{FedMerGut92} (see \cite[Cor. 8.2]{PLG}). 
\end{example}

In addition to providing an understanding of minimax regret against a set of strategies, sequential Rademacher complexity can serve as a starting point for algorithmic development. As shown in \cite{RakShaSri12nips}, any {\em admissible relaxation} can be used to define a succinct algorithm with a regret guarantee. For the setting of this paper, this means the following. Let ${\bf Rel}:\Z^t\mapsto \reals$, for each $t$, be a collection of functions satisfying two conditions:
$$\forall t, ~\inf_{q_t\in \QD}\sup_{z_t\in \Z} \left\{ \Ex_{f_t\sim q_t}\loss(f_t,z_t) + {\bf Rel}(z_{1:t})\right\} \leq {\bf Rel}(z_{1:t-1}), ~\mbox{and }~  -\inf_{\pi\in\Pi} \sum_{t=1}^T \loss(\pi_t(z_{1:t-1}),z_t)\leq {\bf Rel}(z_{1:T}) \ .$$
Then we say that \emph{the relaxation is admissible}. It is then easy to show that regret of any algorithm that ensures above inequalities is bounded by ${\bf Rel}(\{\})$. 
\begin{theorem} \label{theorem:admissible} The conditional sequential Rademacher complexity with respect to $\Pi$
$$\Rad (\ell,\Pi | z_1,\ldots,z_{t}) \deq \sup_{\z,\w} \Ex_{\epsilon_{t+1:T}} \sup_{\pi \in\Pi} \left[ 2 \sum_{s=t+1}^T \epsilon_s\ell(\pi_s((z_{1:t},\w_{1:s-t-1}(\epsilon)),\z_{s-t}(\epsilon)) -  \sum_{s=1}^t \ell(\pi_s(z_{1:s-1}),z_s) \right]$$
is admissible.
\end{theorem}
Conditional sequential Rademacher complexity can therefore be used as a starting point for possibly deriving computationally attractive algorithms, as shown throughout the paper.

We may now define covering numbers for the set $\Pi$ of strategies over the history trees. The development is a straightforward modification of the notions we developed in \citep{RakSriTew10nips}, where we replace ``any tree $\x$'' with a tree of histories $\w_{1:t-1}$.

\begin{definition}
	A set $V$ of $\reals$-valued trees is an \emph{$\alpha$-cover} (with respect to $\ell_p$) of a set of strategies $\Pi$ on an $\mathcal{Z}^*$-valued history tree $\w$ if
	\begin{align}
		\label{eq:lpcover-trees-strategies}
		\textstyle\forall \pi\in\Pi, ~\forall \epsilon\in\{\pm1\}^T, ~~\exists \vv\in V ~~~\mbox{s.t.}~~~ \left(\frac{1}{T}\sum_{t=1}^T |\pi_t(\w_{1:t-1}(\epsilon))-\vv_t(\epsilon)|^p\right)^{1/p}\leq \alpha \ .
	\end{align}
	An \emph{$\alpha$-covering number} $\cN_p(\Pi,\w,\alpha)$ is the size of the smallest $\alpha$-cover.
\end{definition}
For supervised learning, $(\x_{1:t}(\epsilon), \y_{1:t-1}(\epsilon))$ takes place of $\w_{1:t-1}(\epsilon)$.
Now, for any history tree $\w$, sequential Rademacher averages of a class of $[-1,1]$-valued strategies $\Pi$ satisfy
\begin{align*}
	\Rad(\Pi,\w)\leq \inf_{\alpha\geq0}\left\{\alpha T + \sqrt{2\log \cN_1(\Pi,\w, \alpha) T} \right\}
\end{align*}
and the Dudley entropy integral type bound also holds:
\begin{align}
	\label{eq:dudley_wc_history}
	\Rad(\Pi,\w) \le \inf_{\alpha\geq 0}\left\{4 \alpha T + 12\sqrt{T} \int_{\alpha}^{1} \sqrt{\log \ \cN_2(\Pi, \w, \delta ) \ } d \delta \right\}
\end{align}
In particular, this bound should be compared with Theorem 7 in \citep{cesa1999prediction}, which employs a covering number in terms of a pointwise metric between strategies that requires closeness for \emph{all histories and all time steps}. Second, the results of \citep{cesa1999prediction} for real-valued prediction require strategies to be bounded away from $0$ and $1$ by $\delta>0$ and this restriction spoils the rates. 

In the rest of the paper, we show how the results of this section (a) yield proofs of existence of regret-minimization strategies with certain rates and (b) guide in the development of algorithms. For some of these examples, standard methods (such as Exponential Weights) come close to providing an optimal rate, while for others -- fail miserably.

\section{Competing with Autoregressive Strategies}
\label{sec:autoregressive}

In this section, we consider strategies that depend linearly on the past outcomes. To this end, we fix a set $\Theta\subset \reals^k$, for some $k>0$, and parametrize the set of strategies as
$$\textstyle\Pi_\Theta = \left\{\pi^\theta: \pi^\theta_t(z_1,\ldots,z_{t-1})= \sum_{i=0}^{k-1} \theta_{i+1} z_{t-k+i}, ~~~ \theta=(\theta_1,\ldots,\theta_k)\in\Theta\right\}$$
For consistency of notation, we assume that the sequence of outcomes is padded with zeros for $t\leq 0$. First, as an example where known methods \emph{can} recover the correct rate, we consider the case of a constant look-back of size $k$. We then extend the study to cases where neither the regret behavior nor the algorithm is known in the literature, to the best of our knowledge.

\subsection{Finite Look-Back}

Suppose $\cZ=\F\subset \reals^d$ are $\ell_2$ unit balls, the loss is $\loss(f,z) = \inner{f,z}$, and $\Theta\subset \reals^k$ is also a unit $\ell_2$ ball. Denoting by ${\bf W}_{(t-k:t-1)}=[\w_{t-k}(\epsilon),\ldots,\w_{t-1}(\epsilon)]$ a matrix with columns in $\Z$,
\begin{align}
	\label{eq:finite_lookback}
	\Rad(\ell, \Pi_\Theta) &= \sup_{\w,\z}\mathbb{E}_{\epsilon} \sup_{\theta\in\Theta}\left[ \sum_{t=1}^T \epsilon_t \inner{\pi^\theta(\w_{t-k:t-1}(\epsilon)), \z_t(\epsilon)} \right] =\sup_{\w,\z}\mathbb{E}_{\epsilon} \sup_{\theta\in\Theta}\left[ \sum_{t=1}^T \epsilon_t  \z_t(\epsilon)^\tr{\bf W}_{(t-k:t-1)}\cdot\theta \right] \notag\\
	&=\sup_{\w,\z}\mathbb{E}_{\epsilon} \left\|\sum_{t=1}^T \epsilon_t  \z_t(\epsilon)^\tr {\bf W}_{(t-k:t-1)} \right\| \leq \sqrt{kT}
\end{align}
In fact, this bound against all strategies parametrized by $\Theta$ is achieved by the gradient descent (GD) method with the simple update
$\theta_{t+1} = \mbox{Proj}_{\Theta} (\theta_t - \eta \left[ z_{t-k},\dots,z_{t-1} \right]^\tr z_t)$
where $\mbox{Proj}_{\Theta}$ is the Euclidean projection onto the set $\Theta$. This can be seen by writing the loss as
$$\langle \left[ z_{t-k},\dots,z_{t-1} \right]\cdot \theta_t, z_t \rangle = \langle \theta_t, \left[ z_{t-k},\dots,z_{t-1} \right]^\tr z_t \rangle.$$ The regret of GD,
%The cumulative loss of the game is
%$$\sum_{t=1}^{T} \langle \left[ z_{t-k},\dots,z_{t-1} \right]\cdot \theta_t, z_t \rangle =\sum_{t=1}^{T} \langle \theta_t, \left[ z_{t-k},\dots,z_{t-1} \right]^\tr z_t \rangle$$
$\sum_{t=1}^{T} \langle \theta_t, \left[ z_{t-k},\dots,z_{t-1} \right]^\tr z_t \rangle - \inf_{\theta \in \Theta} \sum_{t=1}^{T} \langle \theta, \left[ z_{t-k},\dots,z_{t-1} \right]^\tr z_t \rangle,$
is precisely regret against strategies in $\Theta$, and analysis of GD yields the rate in \eqref{eq:finite_lookback}.

\subsection{Full Dependence on History}

The situation becomes less obvious when $k=T$ and strategies depend on the full history. The regret bound in \eqref{eq:finite_lookback} is vacuous, and the question is whether a better bound can be proved, under some additional assumptions on $\Theta$. Can such a bound be achieved by GD?
	
For simplicity, consider the case of $\F=\Z=[-1,1]$, and assume that $\Theta=B_p(1)\subset\reals^T$ is a unit $\ell_p$ ball, for some $p\geq 1$. Since $k=T$, it is easier to re-index the coordinates so that
$$\textstyle\pi_t^{\theta}(z_{1:t-1}) = \sum_{i=1}^{t-1} \theta_i z_i.$$
The sequential Rademacher complexity of the strategy class is
\begin{align*}
&\Rad(\ell, \Pi_\Theta) = \sup_{\mathbf{w}, \mathbf{z}} \mathbb{E} \sup_{\theta \in \Theta} \left[ \sum_{t=1}^T \epsilon_t  \pi^{\theta}(\w_{1:t-1}(\epsilon))\cdot \z_t(\epsilon)  \right] =\sup_{\w,\z} \mathbb{E} \sup_{\theta \in \Theta} \left[ \sum_{t=1}^T \left( \sum_{i=1}^{t-1} \theta_i \w_i(\epsilon)\right) \epsilon_t \z_t(\epsilon)  \right] \ .
\end{align*}
Rearranging the terms, the last expression is equal to
\begin{align*}
&\sup_{\w,\z} \mathbb{E} \sup_{\theta \in \Theta} \left[ \sum_{t=1}^{T-1} \theta_t \w_t(\epsilon)\cdot \left( \sum_{i=t+1}^{T}  \epsilon_i \z_i(\epsilon) \right) \right] 
%\leq \sup_{\w,\z} \mathbb{E}  \sup_{\theta \in \Theta} \left[ \sum_{t=1}^{T-1} \left| \theta_t \w_t(\epsilon)\right| \cdot \max_{1\leq t\leq T}\left| \sum_{i=t+1}^{T}  \epsilon_i \z_i(\epsilon) \right| \right]\\
\leq \sup_{\w,\z} \mathbb{E}   \left[ \left\| \w_{1:T-1} (\epsilon)\right\|_q \cdot \max_{1\leq t\leq T}\left| \sum_{i=t+1}^{T}  \epsilon_i \z_i(\epsilon) \right| \right]
\end{align*}
where $q$ is the H\"older conjugate of $p$. Observe that
$$\sup_{\z} \mathbb{E} \sup_{1 \leq t \leq T} \left\vert \sum_{i=t}^{T} \epsilon_i \z_i(\epsilon) \right\vert
\leq \sup_{\z} \mathbb{E} \left[\left\vert \sum_{i=1}^{T} \epsilon_i \z_i(\epsilon) \right\vert + \sup_{1 \leq t \leq T} \left\vert \sum_{i=1}^{t-1} \epsilon_i \z_i(\epsilon) \right\vert \right]
\leq 2 \sup_{\z} \mathbb{E} \sup_{1 \leq t \leq T} \left\vert \sum_{i=1}^{t} \epsilon_i \z_i(\epsilon) \right\vert$$
Since $\{\epsilon_t \z_t(\epsilon): t = 1, \dots, T\}$ is a bounded martingale difference sequence, the last term is of the order of $\cO(\sqrt{T})$.
Now, suppose there is some $\beta>0$ such that $\left\| \w_{1:T-1} (\epsilon)\right\|_q\leq T^\beta$ for all $\epsilon$. This assumption can be implemented if we consider constrained adversaries, where such $\ell_q$-bound is required to hold for any prefix $\w_{1:t} (\epsilon)$ of history (In Appendix, we prove Theorem~\ref{thm:compete_with_strategies} for the case of constrained sequences). Then $\Rad(\ell, \Pi_\Theta) \leq C \cdot T^{\beta+1/2}$ for some constant $C$. We now compare the rate of convergence of sequential Rademacher and the rate of the mirror descent algorithm for different settings of $q$ in Table \ref{table:ratecomparison}. If $\Vert\theta\Vert_{p} \leq 1$ and $\Vert\w\Vert_{q} \leq T^{\beta}$ for $q \geq 2$, the convergence rate of mirror descent with Legendre function $F(\theta) = \frac{1}{2} \Vert \theta \Vert^2_p$ is $\sqrt{q-1} T^{\beta+1/2}$ (see \citep{SreSriTew11}).

\begin{table}[ht]
\centering
\begin{tabular}{c c  c  c  c c}
 & $\Theta$ & $\w_{1:T}$ & sequential Radem. rate & Mirror descent rate \\
\hline
 & $B_1(1)$ & $\Vert\w_{1:T-1}\Vert_{\infty} \leq 1$ & $\sqrt{T}$ & $\sqrt{T \log T}$ \\
\hline
$q \geq 2$ & $B_p(1)$  & $\Vert\w_{1:T-1}\Vert_{q} \leq T^{\beta}$ & $T^{\beta+1/2}$ & $\sqrt{q-1} T^{\beta+1/2}$ \\
\hline
 & $B_2(1)$ & $\Vert\w_{1:T-1}\Vert_{2} \leq T^{\beta}$ & $T^{\beta+1/2}$ & $T^{\beta+1/2}$ \\
\hline
$1 \leq q \leq 2$ &$B_p(1)$  & $\Vert\w_{1:T-1}\Vert_{q} \leq T^{\beta}$ & $T^{\beta+1/2}$ & $T^{\beta+1/q}$ \\
\hline
 & $B_{\infty}(1)$ & $\Vert\w_{1:T-1}\Vert_{1} \leq T^{\beta}$ & $T^{\beta+1/2}$ & T \\
\end{tabular}
\label{table:ratecomparison}
\caption{Comparison of the rates of convergence (up to constant factors)}
\end{table}

We observe that mirror descent, which is known to be optimal for online linear optimization, and which gives the correct rate for the case of \emph{bounded} look-back strategies, in several regimes fails to yield the correct rate for more general linearly parametrized strategies. Even in the most basic regime where $\Theta$ is a unit $\ell_1$ ball and the sequence of data is not constrained (other than $\Z=[-1,1]$), there is a gap of $\sqrt{\log T}$ between the Rademacher bound and the guarantee of mirror descent. Is there an algorithm that removes this factor?

\subsubsection{Algorithms for $\Theta=B_1(1)$}

For the example considered in the previous section, with $\F=\Z=[-1,1]$ and $\Theta=B_1(1)$, the conditional sequential Rademacher complexity of Theorem~\ref{theorem:admissible} becomes
\begin{align*}
 \Rad_T (\Pi | z_1,\ldots,z_{t}) &= \sup_{\z,\w} \Ex_{\epsilon_{t+1:T}} \sup_{\pi \in\Pi} \left[ 2 \sum_{s=t+1}^T \epsilon_s \pi_s(z_{1:t},\w_{1:s-t-1}(\epsilon)) \cdot \z_s(\epsilon)  -  \sum_{s=1}^t  \pi_s(z_{1:s-1})\cdot z_s  \right] \\
& \leq \sup_{\w} \Ex_{\epsilon_{t+1:T}} \sup_{\pi \in \Pi} \left[ 2 \sum_{s=t+1}^T \epsilon_s \pi_s(z_{1:t},\w_{1:s-t-1}(\epsilon)) - \sum_{s=1}^t z_s \pi_s(z_{1:s-1}) \right]
\end{align*}
where the $\z$ tree is ``erased'', as at the end of the proof of Theorem~\ref{thm:main_lip_contraction}. Define $a_s(\epsilon) = 2 \epsilon_s$ for $s>t$ and $-z_s$ otherwise; $b_i(\epsilon) = \w_i(\epsilon)$ for $i>t$ and $z_i$ otherwise. We can then simply write
\begin{align*}
\sup_{\w} &\Ex_{\epsilon_{t+1:T}}\sup_{\theta \in \Theta} \left[ \sum_{s=1}^T a_s(\epsilon) \sum_{i=1}^{s-1} \theta_i b_i(\epsilon) \right] = \sup_{\w} \Ex_{\epsilon_{t+1:T}}\sup_{\theta \in \Theta} \left[ \sum_{s=1}^{T-1} \theta_s b_s(\epsilon) \sum_{i=s+1}^{T} a_i(\epsilon) \right]
%&= \sup_{\w} \Ex_{\epsilon_{t+1:T}} \sup_{\theta \in \Theta} \langle (\theta_0 b_0(\epsilon),\dots, \theta_{T-1} b_{T-1}(\epsilon)), (\sum_{i=1}^T a_i(\epsilon), \dots, a_T(\epsilon)) \rangle
\leq \Ex_{\epsilon_{t+1:T}} \max_{1 \leq s \leq T} \left| \sum_{i=s}^T a_i(\epsilon) \right|
\end{align*}
which we may use as a relaxation: 
\begin{lemma}
	\label{lem:admissibility_of_lin_l1_linf}
	Define $a^t_s(\epsilon) = 2 \epsilon_s$ for $s>t$, and $-z_s$ otherwise. Then,
$$\textstyle{\bf Rel}(z_{1:t}) = \Ex_{\epsilon_{t+1:T}} \max_{1 \leq s \leq T} \left| \sum_{i=s}^T a^t_i(\epsilon) \right|$$
is an admissible relaxation.
\end{lemma}
With this relaxation, the following method attains $\cO(\sqrt{T})$ regret: prediction at step $t$ is
$$q_t = \argmin{q \in [-1,1]} \sup_{z_t \in \{\pm 1\}} \left\{ \Ex_{f_t \sim q} f_t \cdot z_t + \mathbb{E}_{\epsilon_{t+1:T}} \max_{1 \leq s \leq T} \left| \sum_{i=s}^T a^t_i(\epsilon) \right| \right\}$$
where the sup over $z_t\in[-1,1]$ is achieved at $\{\pm 1\}$ due to convexity. Following \citep{RakShaSri12nips}, we can also derive randomized algorithms, which can be viewed as ``randomized playout'' generalizations of the Follow the Perturbed Leader algorithm.

\begin{lemma} 
	\label{lem:treewalk}
Consider the randomized strategy where at round $t$ we first draw $\epsilon_{t+1},\ldots,\epsilon_T$ uniformly at random and then further draw our move $f_t$ according to the distribution
\begin{align*}
	q_t(\epsilon) &=\textstyle\argmin{q \in [-1,1]} \sup_{z_t \in \{-1,1\}}\left\{ \Ex_{f_t \sim q} f_t \cdot z_t + \max_{1 \leq s \leq T} \left| \sum_{i=s}^T a^t_i(\epsilon) \right| \right\} \\
	%&=\frac{1}{2} \left( \norm{(\ldots, -\sum_{s=d}^{t-1} z_s + 1 + 2 \sum_{s=t+1}^T \epsilon_s, \ldots, 1+ 2 \sum_{s=t+1}^T \epsilon_s, 2 \sum_{s=t+1}^T \epsilon_s, \ldots \dots) }_{\infty}  - \norm{ (\dots, -\sum_{s=d}^{t-1} z_s -1 + 2 \sum_{s=t+1}^T \epsilon_s,\dots) }_{\infty} \right)
	&=\textstyle\frac{1}{2}\left(~\max\left\{\max_{s=1,\ldots,t}\left| -\sum_{i=s}^{t-1} z_i +1+2\sum_{i=t+1}^T\epsilon_i\right|~,~~ \max_{s=t+1,\ldots,T}\left| 2\sum_{i=s}^T\epsilon_i \right|\right\} \right.\\ 
	&\textstyle\left.~~~~ -\max\left\{\max_{s=1,\ldots,t}\left|-\sum_{i=s}^{t-1} z_i-1+2\sum_{i=t+1}^T\epsilon_i\right|~,~~ \max_{s=t+1,\ldots,T}\left| 2\sum_{i=s}^T\epsilon_i \right|\right\}\right)
\end{align*}
The expected regret of this randomized strategy is upper bounded by sequential Rademacher complexity:
$\E{\Reg_T} \le 2 \Rad_T(\Pi)$, which was shown to be $\cO(\sqrt{T})$ (see Table~\ref{table:ratecomparison}).
\end{lemma}
The time consuming parts of the above randomized method are to draw $T-t$ random bits at round $t$ and to calculate the partial sums. However, we may replace Rademacher random variables by Gaussian $\N(0,1)$ random variables and use known results on the distributions of extrema of a Brownian motion. To this end, define a Gaussian analogue of conditional  sequential Rademacher complexity
$$ \G_T (\Pi | z_1,\ldots,z_{t}) = \sup_{\z,\w} \Ex_{\sigma_{t+1:T}} \sup_{\pi \in\Pi} \left[ \sqrt{2\pi} \sum_{s=t+1}^T \sigma_s\ell(\pi_s((z_{1:t},\w_{1:s-t-1}(\epsilon)),\z_{s-t}(\epsilon)) -  \sum_{s=1}^t \ell(\pi_s(z_{1:s-1}),z_s) \right]$$
where $\sigma_t \sim \N(0,1)$, and $\epsilon = (\sign(\sigma_1),\dots,\sign(\sigma_T))$. For our example the $\cO(\sqrt{T})$ bound can be shown for $\G_T (\Pi)$ by calculating the expectation of the maximum of Brownian motion. Proofs similar to  Theorem~\ref{thm:compete_with_strategies} and Theorem~\ref{theorem:admissible} show that the conditional Gaussian complexity $\G_T (\Pi | z_1,\ldots,z_{t})$ is an upper bound on $\Rad_T(\Pi | z_1,\ldots,z_{t})$ and is admissible (see Theorem~\ref{theorem:gaussian_admissible} in Appendix). Furthermore, the proof of Lemma~\ref{lem:treewalk} holds for Gaussian random variables, and gives the randomized algorithm as in Lemma~\ref{lem:treewalk} with $\epsilon_t$ replaced by $\sigma_t$. 
It is not difficult to see that we can keep track of the maximum and minimum of $\left\{-\sum_{i=s}^{t-1} z_i \right\}$ between rounds in $\cO(1)$ time. We can then draw three random variables from the joint distribution of the maximum, the minimum and the endpoint of a Brownian Motion and calculate the prediction in $\cO(1)$ time per round of the game (the joint distribution can be found in \cite{Karatzas:Brownian}). In conclusion, we have derived an algorithm that for the case of $\Theta=B_1(1)$, with time complexity of $\cO(1)$ per round and the optimal regret bound of $\cO(\sqrt{T})$. We leave it as an open question to develop efficient and optimal algorithms for the other settings in Table~\ref{table:ratecomparison}.

\section{Competing with Statistical Models}
\label{sec:statmodels}

In this section we consider competing with a set of strategies that arise from statistical models. For example, for the case of Bayesian models, strategies are parametrized by the choice of a prior. Regret bounds with respect to a set of such methods can be thought of as a robustness statement: we are aiming to perform as well as the strategy with the best choice of a prior. We start this section with a general setup that needs further investigation.

\subsection{Compression and Sufficient Statistics}
Assume that strategies in $\Pi$ have a particular form: they all work with a ``sufficient statistic'', or, more loosely, \emph{compression} of the past data. Suppose ``sufficient statistics'' can take values in some set $\Gamma$. Fix a set $\bar{\Pi}$ of mappings $\bar{\pi}:\Gamma\mapsto \F$. We assume that all the strategies in $\Pi$ are of the form
$\pi_t(z_1,\ldots,z_{t-1}) = \bar{\pi}(\gamma(z_1,\ldots,z_{t-1}))$ for some $\bar{\pi}\in\bar{\Pi}$
and $\gamma:\cZ^*\mapsto \Gamma$. Such a bottleneck $\Gamma$ can arise due to a finite memory or finite precision, but can also arise if the strategies in $\Pi$ are actually solutions to a statistical problem. If we assume a certain stochastic source for the data, we may estimate the parameters of the model, and there is often a natural set of sufficient statistics associated with it. If we collect all such solutions to stochastic models in a set $\Pi$, we may compete with all these strategies as long as $\Gamma$ is not too large and the dependence of estimators on these sufficient statistics is smooth. With the notation introduced in this paper, we need to study the sequential Rademacher complexity for strategies $\Pi$, which can be upper bounded by the complexity of $\bar{\Pi}$ on $\Gamma$-valued trees:
$$\Rad(\Pi)\leq \sup_{\g,\z}\mathbb{E}_{\epsilon} \sup_{\bar{\pi} \in \bar{\Pi}}\left[ \sum_{t=1}^T \epsilon_t \loss(\bar{\pi}(\g_t(\epsilon)), \z_t(\epsilon) \right]$$
This complexity corresponds to our intuition that with sufficient statistics the dependence on the ever-growing history can be replaced with the dependence on a \emph{summary} of the data. Next, we consider one particular case of this general idea, and refer to \cite{FosRakSriTew11} for more details on these types of bounds.

\subsection{Bernoulli Model with a Beta Prior}
\label{sec:bernoulli_beta}
Suppose the data $z_t \in \{0,1\}$ is generated according to Bernoulli distribution  with parameter $p$, and the prior on $p \in [0,1]$ is $p \sim Beta(\alpha,\beta)$. Given the data $\{z_1,\ldots,z_{t-1}\}$, the  maximum a posteriori (MAP) estimator of $p$ is
$\hat{p} = (\sum_{i=1}^{t-1} z_i + \alpha - 1)/(t-1 + \alpha + \beta -2)$. We now consider the problem of competing with $\Pi=\{\pi^{\alpha,\beta}: \alpha>1, \beta\in(1,C_\beta]\}$ for some $C_\beta$, where each $\pi^{\alpha,\beta}$ predicts the corresponding MAP value for the next round:
$$\textstyle\pi^{\alpha,\beta}_t (z_1,\ldots,z_{t-1}) = (\sum_{i=1}^{t-1} z_i + \alpha - 1)/(t-1 + \alpha + \beta -2) \ .$$
Let us consider the absolute loss, which is equivalent to probability of a mistake of the randomized prediction\footnote{Alternatively, we can consider strategies that predict according to $\ind{\hat{p}\geq 1/2}$, which better matches the choice of an absolute loss. However, in this situation, an experts algorithm on an appropriate discretization attains the bound.} with bias $\pi^{\alpha,\beta}_t$. Thus, the loss of a strategy $\pi^{\alpha,\beta}$ on round $t$ is $ \left| \pi^{\alpha,\beta}_t(z_{1:t-1}) - z_t\right|\ .$
Using Theorem~\ref{thm:compete_with_strategies} and the argument in \eqref{eq:contraction_indicator} to erase the outcome tree, we conclude that there exists a regret minimization algorithm against the set $\Pi$ which attains regret of at most
$$\textstyle 2 \sup_{\w} \mathbb{E}_{\epsilon} \sup_{\alpha,\beta} \left[ \sum_{t=1}^T \epsilon_t \frac{ \sum_{i=1}^{t-1} \w_i(\epsilon) + \alpha -1 }{t-1+\alpha+\beta-2} \right] \ .$$
To analyze the rate exhibited by this upper bound, construct a new tree with $\g_1(\epsilon) = 1$ and 
$\g_t(\epsilon) = \frac{\sum_{i=1}^{t-1} \w_i(\epsilon)+\alpha-1}{t+\alpha-2} \in [0,1]$ for $t \geq 2$. With this notation, we can simply re-write the last expression as twice
\begin{align*}
 \textstyle\sup_{\g} \mathbb{E}_{\epsilon} \sup_{\alpha, \beta} \left[ \sum_{t=1}^T \epsilon_t  \g_t(\epsilon) \frac{t+\alpha-2}{t+\alpha+\beta-3 }\right]
\end{align*}
The supremum ranges over all $[0,1]$-valued trees $\g$, but we can pass to the supremum over all $[-1,1]$-valued trees (thus making the value larger). We then observe that the supremum is achieved at a $\{\pm1\}$-valued tree $\g$, 
which can then be erased as in the end of the proof of Theorem~\ref{thm:main_lip_contraction} (roughly speaking, it amounts to renaming $\epsilon_t$ into $\epsilon_t\g_t(\epsilon_{1:t-1})$). We obtain an upper bound
\begin{align}
	\label{eq:rate_alpha_beta}
	\Rad(\Pi) \leq \mathbb{E}_{\epsilon} \sup_{\alpha,\beta} \sum_{t=1}^T \frac{\epsilon_t  (t+\alpha-2) }{t+\alpha+\beta-3} 	
	\leq  \mathbb{E}_{\epsilon} \left\vert \sum_{t=1}^T \epsilon_t \right\vert+ \mathbb{E}_{\epsilon} \sup_{\alpha,\beta} \left\vert \sum_{t=1}^T \frac{\epsilon_t  (\beta-1) }{t+\alpha+\beta-3} \right\vert = (\sqrt{C_\beta} + 1)\sqrt{T}
\end{align}
where we used Cauchy-Schwartz inequality for the second term. We note that an experts algorithm would require a discretization that depends on $T$ and will yield a regret bound of order $O(\sqrt{T\log T})$. It is therefore interesting to find an algorithm that avoids the discretization and obtains this regret. To this end, we take the derived upper bound on the sequential Rademacher complexity and prove that it is an admissible relaxation.

\begin{lemma}
	\label{lem:alpha_beta_relaxation_admissible}
	The relaxation
	$${\bf Rel}(z_{1:t}) = \mathbb{E}_{\epsilon_{t+1:T}} \sup_{\alpha, \beta} \left[ 2 \sum_{s=t+1}^T\epsilon_s \cdot \frac{s+\alpha-2}{s+\alpha+\beta-3} - \sum_{s=1}^t \left| \frac{\sum_{i=1}^{s-1} z_i}{s+\alpha+\beta -3} -z_s \right| \right]$$
	is admissible.
\end{lemma}
Given that this relaxation is admissible, we have a guarantee that the following algorithm attains the rate $(\sqrt{C_\beta} + 1)\sqrt{T}$ given in \eqref{eq:rate_alpha_beta}:
\begin{align*}
q_t = & \arg\min_{q\in[0,1]} \max_{z_t\in\{0,1\}} \left\{\mathbb{E}_{f \sim q} \vert f-z_t \vert +  \mathbb{E}_{\epsilon_{t+1:T}} \sup_{\alpha, \beta} \left[ 2 \sum_{s=t+1}^T\epsilon_s \cdot \frac{s+\alpha-2}{s+\alpha+\beta-3} - \sum_{s=1}^t \left| \frac{\sum_{i=1}^{s-1} z_i}{s+\alpha+\beta -3} -z_s \right| \right] \right\}
\end{align*}
In fact, $q_t$ can be written as 
\begin{align*}
q_t&= %\frac{1}{2} \left\{  \mathbb{E}_{\epsilon_{t+1:T}} \sup_{\alpha,\beta} \left[ 2 \sum_{s=t+1}^T\epsilon_s \cdot \frac{s+\alpha-2}{s+\alpha+\beta-3} - \sum_{s=1}^{t-1} \left| \frac{\sum_{i=1}^{s-1} z_i}{s+\alpha+\beta-3} -z_s \right| + \frac{\sum_{i=1}^{t-1} z_i}{t+\alpha+\beta-3} \right]  \right.\\
%&\left.~- \mathbb{E}_{\epsilon_{t+1:T}} \sup_{\alpha,\beta} \left[ 2 \sum_{s=t+1}^T\epsilon_s \cdot \frac{s+\alpha-2}{s+\alpha+\beta-3} - \sum_{s=1}^{t-1} \left| \frac{\sum_{i=1}^{s-1} z_i}{s+\alpha+\beta-3} -z_s \right| - \frac{\sum_{i=1}^{t-1} z_i}{t+\alpha+\beta-3}\right] \right\} \\
\frac{1}{2} \left\{  \mathbb{E}_{\epsilon_{t+1:T}} \sup_{\alpha,\beta} \left[ 2 \sum_{s=t+1}^T\epsilon_s \cdot \frac{s+\alpha-2}{s+\alpha+\beta-3} - \sum_{s=1}^{t-1} (1- 2 z_s) \cdot \frac{\sum_{i=1}^{s-1} z_i}{s+\alpha+\beta-3}  + \frac{\sum_{i=1}^{t-1} z_i}{t+\alpha+\beta-3} \right]
\right.\\
&\left.~~~~- \mathbb{E}_{\epsilon_{t+1:T}} \sup_{\alpha,\beta} \left[ 2 \sum_{s=t+1}^T\epsilon_s \cdot \frac{s+\alpha-2}{s+\alpha+\beta-3} - \sum_{s=1}^{t-1} (1 - 2 z_s) \cdot \frac{\sum_{i=1}^{s-1} z_i}{s+\alpha+\beta-3} - \frac{\sum_{i=1}^{t-1} z_i}{t+\alpha+\beta-3} \right] \right\}
\end{align*}
For a given realization of random signs, the supremum is an optimization of a sum of linear fractional functions of two variables. Such an optimization can be carried out in time $O(T\log T)$ (see \cite{chen2005efficient}). To deal with the expectation over random signs, one may either average over many realizations or use the random playout idea and only draw one sequence. Such an algorithm is admissible for the above relaxation, obtains the $O(\sqrt{T})$ bound, and runs in $O(T\log T)$ time per step. We leave it as an open problem whether a more efficient algorithm with $O(\sqrt{T})$ regret exists.

\section{Competing with Regularized Least Squares}
\label{sec:rls}

Consider the supervised learning problem with $\Y=[-1,1]$ and some set $\X$. Consider the Regularized Least Squares (RLS) strategies, parametrized by a regularization parameter $\lambda$ and a shift $w_0$. That is, given data $(x_1,y_1),\ldots,(x_{t},y_{t})$, the strategy solves
$$\textstyle\arg\min_{w} \sum_{i=1}^{t} (y_i-\inner{x_i,w})^2 + \lambda \|w-w_0\|^2 \ .$$
For a given pair $\lambda$ and $w_0$, the solution is
$$w_{t+1}^{\lambda,w_0} = w_0 + (X^\tr X+ \lambda I)^{-1} X^\tr Y,$$
where $X \in \mathbb{R}^{t \times d}$ and $Y \in \mathbb{R}^{t \times 1}$ are the usual matrix representations of the data $x_{1:t},y_{1:t}$. We would like to compete against a set of such RLS strategies which make prediction $\inner{w^{\lambda,w_0}_{t-1},x_t}$, given side information $x_t$. Since the outcomes are in $[-1,1]$, without loss of generality we clip the predictions of strategies to this interval, thus making our regret minimization goal only harder. To this end, let $c(a) = a$ if $a\in[-1,1]$ and $c(a)=\sign(a)$ for $|a|>1$. Thus, given side-information $x_t\in\X$, the prediction of strategies in 
$\Pi = \left\{\pi^{\lambda,w_0}: \lambda\geq\lambda_{\min}>0, \|w_0\|_2\leq 1 \right\}$ is simply the clipped product
$$\pi_t^{\lambda,w_0} (x_{1:t},y_{1:t-1}) = c\left( \inner{w_{t-1}^{\lambda,w_0}, x_t} \right) \ .$$ 
Let us take the squared loss function $\loss(\hat{y},y) = \left(\hat{y} - y\right)^2$. 
\begin{lemma}
	\label{lem:rls}
		For the set $\Pi$ of strategies defined above, the minimax regret of competing against Regularized Least Squares strategies is
	\begin{align*}
		\Val_T(\Pi) & \leq c\sqrt{T\log (T\lambda_{\min}^{-1})}
	\end{align*}
	for an absolute constant $c$.
\end{lemma}
Observe that $\lambda_{\min}^{-1}$ enters only logarithmically, which allows us to set, for instance, $\lambda_{\min}=1/T$. 
Finally, we mention that the set of strategies includes $\lambda=\infty$. This setting corresponds to a static strategy $\pi_t^{\lambda,w_0}(x_{1:t},y_{1:t-1}) = \inner{w_0, x_t}$ and regret against such a static family parametrized by $w_0\in B_2(1)$ is exactly the  objective of online linear regression \citep{Vovk98}. Lemma~\ref{lem:rls} thus shows that it is possible to have vanishing regret with respect to a much larger set of strategies. It is an interesting open question of whether one can develop an efficient algorithm with the above regret guarantee.

\section{Competing with Follow the Regularized Leader Strategies}
\label{sec:ftrl}

Consider the problem of online linear optimization with the loss function $\ell(f_t,x_t)  = \ip{f_t}{z_t}$ for $f_t\in\F$, $z_t\in\Z$. For simplicity, assume that $\F=\Z=B_2(1)$. An algorithm commonly used for online linear and online convex optimization problems is the Follow the Regularized Leader (FTRL) algorithm. We now consider competing with a family of FTRL algorithms $\pi^{w_0,\lambda}$ indexed by $w_0 \in \{ w : \norm{w}\le 1\}$ and $\lambda \in \Lambda$ where $\Lambda$ is a family of functions $\lambda : \reals^+ \times [T] \mapsto \reals^+$ specifying a schedule for the choice of regularization parameters. Specifically we consider strategies $\pi^{w_0,\lambda}$ such that $\pi^{w_0,\lambda}_t(z_1,\ldots,z_{t-1}) = w_t$ where 
\begin{align}\label{eq:ftrlstrat}
\textstyle w_t =  w_0 + \argmin{w : \norm{w} \le 1}\left\{ \sum_{i=1}^{t-1} \ip{w}{z_i} + \frac{1}{2}\lambda\left(\norm{\sum_{i=1}^{t-1} z_i} , t\right) \norm{w}^2 \right\}
\end{align}
This can be written in closed form as
$w_t = w_0 - (\sum_{i=1}^{t-1} z_i)/\max\left\{\lambda\left(\norm{\sum_{i=1}^{t-1} z_i} , t\right),  \norm{\sum_{i=1}^{t-1} z_i}\right\}$.

\begin{lemma}\label{lem:valftrl}
	For a given class $\Lambda$ of functions indicating choices of the regularization parameters, define a class 
	$\Gamma$ of functions on $[0,1] \times [1/T,1]$ specified by
	$$
	\Gamma = \left\{\gamma : \forall b \in [1/T,1], a \in [0,1],  \gamma(a,b) = \min\left\{ \frac{a/(b-1)}{\lambda(a/(b-1),1/b)} , 1\right\}, \lambda \in \Lambda\right\}
	$$
	Then the value of the online learning game competing against FTRL strategies given by Equation \ref{eq:ftrlstrat} is bounded as
\begin{align*}
\Val_T(\Pi_\Lambda) & \leq 4\ \sqrt{T} + 2\ \mathcal{R}_T(\Gamma)
\end{align*}
where $\mathcal{R}_T(\Gamma)$ is the sequential Rademacher complexity \citep{RakSriTew10nips} of $\Gamma$. 
\end{lemma}

Notice that if $|\Lambda| < \infty$ then the second term is bounded as $\mathcal{R}_T(\Gamma) \le \sqrt{T \log |\Lambda|}$. However, we may compete with an infinite set of step-size rules. Indeed, each $\gamma \in \Gamma$ is a function $[0,1]^2 \mapsto [0,1]$. Hence, even if one considers $\Gamma$ to be the set of all $1$-Lipschitz functions (Lipschitz w.r.t., say, $\ell_\infty$ norm), it holds that
$
 \mathcal{R}_T(\Gamma)  \le 2 \sqrt{T \log T} \ .
$
We conclude that it is possible to compete with set of FTRL strategies that pick any $w_0$ in unit ball as starting point and further use for regularization parameter schedule any $\lambda : \reals^2 \mapsto \reals$ that is such that $\frac{a/(b-1)}{\lambda(a/(b-1),1/b)}$ is a $1$-Lipchitz function for every $a, b \in [1/T,1]$.

Beyond the finite and Lipschitz cases shown above, it would be interesting to analyze richer families of step size schedules, and possibly derive efficient algorithms.

\bibliographystyle{plain}
\bibliography{paper}

\begin{thebibliography}{10}

\bibitem{BenPalSha09}
S.~Ben-David, D.~Pal, and S.~Shalev-Shwartz.
\newblock Agnostic online learning.
\newblock In {\em Proceedings of the 22th Annual Conference on Learning
  Theory}, 2009.

\bibitem{cesa1999prediction}
N.~Cesa-Bianchi and G.~Lugosi.
\newblock On prediction of individual sequences.
\newblock {\em The Annals of Statistics}, 27(6):pp. 1865--1895, 1999.

\bibitem{PLG}
N.~Cesa-Bianchi and G.~Lugosi.
\newblock {\em Prediction, Learning, and Games}.
\newblock Cambridge University Press, 2006.

\bibitem{chen2005efficient}
D.Z. Chen, O.~Daescu, Y.~Dai, N.~Katoh, X.~Wu, and J.~Xu.
\newblock Efficient algorithms and implementations for optimizing the sum of
  linear fractional functions, with applications.
\newblock {\em Journal of Combinatorial Optimization}, 9(1):69--90, 2005.

\bibitem{FedMerGut92}
M.~Feder, N.~Merhav, and M.~Gutman.
\newblock Universal prediction of individual sequences.
\newblock {\em Information Theory, IEEE Transactions on}, 38(4):1258--1270,
  1992.

\bibitem{FosRakSriTew11}
D.~P. Foster, A.~Rakhlin, K.~Sridharan, and A.~Tewari.
\newblock Complexity-based approach to calibration with checking rules.
\newblock {\em Journal of Machine Learning Research - Proceedings Track},
  19:293--314, 2011.

\bibitem{Karatzas:Brownian}
I.~Karatzas and S.~E. Shreve.
\newblock {\em Brownian Motion and Stochastic Calculus}.
\newblock Springer-Verlag, Berlin, 2nd edition, 1991.

\bibitem{MerFed98}
N.~Merhav and M.~Feder.
\newblock Universal prediction.
\newblock {\em IEEE Transactions on Information Theory}, 44:2124--2147, 1998.

\bibitem{RakShaSri12nips}
A.~Rakhlin, O.~Shamir, and K.~Sridharan.
\newblock Relax and randomize : From value to algorithms.
\newblock In {\em Advances in Neural Information Processing Systems}, 2012.

\bibitem{rakhlin:notes}
A.~Rakhlin and K.~Sridharan.
\newblock
  \url{http://www-stat.wharton.upenn.edu/~rakhlin/courses/stat928/stat928_notes.pdf},
  2012.

\bibitem{RakSriTew10nips}
A.~Rakhlin, K.~Sridharan, and A.~Tewari.
\newblock Online learning: Random averages, combinatorial parameters, and
  learnability.
\newblock In {\em Advances in Neural Information Processing Systems}, 2010.

\bibitem{RakSriTew11nips}
A.~Rakhlin, K.~Sridharan, and A.~Tewari.
\newblock Online learning: Stochastic, constrained, and smoothed adversaries.
\newblock In {\em NIPS}, pages 1764--1772, 2011.

\bibitem{SreSriTew11}
N.~Srebro, K.~Sridharan, and A.~Tewari.
\newblock On the universality of online mirror descent.
\newblock In {\em NIPS}, pages 2645--2653, 2011.

\bibitem{Vovk98}
V.~Vovk.
\newblock Competitive on-line linear regression.
\newblock In {\em NIPS '97: Proceedings of the 1997 conference on Advances in
  neural information processing systems 10}, pages 364--370, Cambridge, MA,
  USA, 1998. MIT Press.

\end{thebibliography}

\newpage
\appendix

\section{Proofs}

\begin{proof}[\textbf{Proof of Theorem~\ref{thm:compete_with_strategies}}]
	Let us prove a more general version of Theorem~\ref{thm:compete_with_strategies}, which we do not state in the main text due to lack of space. The extra twist is that we allow constraints on the sequences $z_1,\ldots,z_T$ played by the adversary. Specifically, the adversary at round $t$ can only play $x_t$ that satisfy constraint $C_t(z_1,\ldots,z_t) = 1$ where $(C_1,\ldots,C_T)$ is a predetermined sequence of constraints with $C_t : \Z^t \mapsto \{0,1\}$. When each $C_t$ is the function that is always $1$ then we are in the setting of the theorem statement where we play an unconstrained/worst case adversary. However the proof here allows us to even analyze constrained adversaries which come in handy in many cases. Following \citep{RakSriTew11nips}, a {\em restriction} $\PD_{1:T}$ on the adversary is a sequence $\PD_1,\ldots,\PD_T$ of mappings $\PD_t: \Z^{t-1}\mapsto 2^\PD$ such that $\PD_t(z_{1:t-1})$ is a \emph{convex} subset of $\PD$ for any $z_{1:t-1}\in\Z^{t-1}$. In the present proof we will only consider {\em constrained adversaries}, where $\PD_t = \Delta({\mathcal C}_t(z_{1:t-1}) )$ the set of all distributions on the constrained subset
	$${\mathcal C}_t(z_{1:t-1}) ~\deq~ \{z \in \Z : C_t(z_1,\ldots,z_{t-1},z) = 1 \} .$$
defined at time $t$ via a binary constraint $C_t: \Z^t \mapsto \{0,1\}$. Notice that the set ${\mathcal C}_t(z_{1:t-1})$ is the subset of $\Z$ from which the adversary is allowed to pick instance $z_t$ from given the history so far. It was shown in \cite{RakSriTew11nips} that such constraints can model sequences with certain properties, such as slowly changing sequences, low-variance sequences, and so on. Let ${\bf C}$ be the set of $\Z$-valued trees $\z$ such that for every $\epsilon\in\{\pm1\}^T$ and $t\in[T]$, $$C_t(\z_1(\epsilon),\ldots,\z_t(\epsilon)) = 1,$$
that is, the set of trees such that the constraint is satisfied along any path. The statement we now prove is that the value of the prediction problem with respect to a set $\Pi$ of strategies and against constrained adversaries (denoted by $\Val_T(\Pi, {\mathcal C}_{1:T})$) is upper bounded by twice the sequential complexity
	\begin{align}
		\label{eq:constained_rademacher_bound}
		\sup_{\w \in {\bf C}, \z}\En_{\epsilon} \sup_{\pi \in \Pi} \sum_{t=1}^T  \epsilon_t \loss(\pi_t(\w_{1}(\epsilon),\ldots,\w_{t-1}(\epsilon))), \z_t(\epsilon))
	\end{align}
	where it is crucial that the $\w$ tree ranges over trees that respect the constraints along all paths, while $\z$ is allowed to be an arbitrary $\Z$-valued tree. This fact that $\w$ respects the constraints is the only difference with the original statement of Theorem~\ref{thm:compete_with_strategies} in the main body of the paper.
	
For ease of notation we use $\multiminimax{\ }_{t=1}^T$ to denote repeated application of operators such has $\sup$ or $\inf$. For instance, $\multiminimax{\sup_{a_t \in A} \inf_{b_t \in B} \En_{r_t \sim P} }_{t=1}^T[F(a_1,b_1,r_1,...,a_T,b_T,r_T)]$ denotes 
$$\sup_{a_1 \in A} \inf_{b_1 \in B} \En_{r_1 \sim P} \ldots \sup_{a_T \in A} \inf_{b_T \in B} \En_{r_T \sim P}[F(a_1,b_1,r_1,...,a_T,b_T,r_T)]$$

	%We now use the $\multiminimax{\ }_{t=1}^T$ notation to denote the repeated sequence of infima, suprema, and expectations. 
	
	The value of a prediction problem with respect to a set of strategies and against constrained adversaries can be written as :
	\begin{align*}
	\Val_T(\Pi, {\mathcal C}_{1:T})&= \multiminimax{\inf_{q_t\in\QD}\sup_{p_t\in \PD_t(z_{1:t-1})}\Eunderone{f_t\sim q_t, z_t\sim p_t}}_{t=1}^T \left[ \sum_{t=1}^T \loss(f_t,z_t)- \inf_{\pi \in \Pi}  \loss(\pi_t(z_{1:t-1}),z_t)\right] \\
	&=\multiminimax{\sup_{p_t\in \PD_t(z_{1:t-1})}\Eunderone{z_t\sim p_t}}_{t=1}^T\sup_{\pi \in \Pi} \left[ \sum_{t=1}^T \inf_{f_t\in\F} \mathbb{E}_{z'_t} \loss(f_t,z'_t)- \loss(\pi_t(z_{1:t-1}),z_t)\right] \\
	&\leq \multiminimax{\sup_{p_t\in \PD_t(z_{1:t-1})}\Eunderone{z_t\sim p_t}}_{t=1}^T \sup_{\pi \in \Pi} \left[ \sum_{t=1}^T \mathbb{E}_{z'_t} \loss(\pi_t(z_{1:t-1}),z'_t)- \loss(\pi_t(z_{1:t-1}),z_t)\right]\\
	&\leq \multiminimax{\sup_{p_t\in \PD_t(z_{1:t-1})}\Eunderone{z_t,z'_t}}_{t=1}^T \sup_{\pi \in \Pi} \left[ \sum_{t=1}^T  \loss(\pi_t(z_{1:t-1}),z'_t)- \loss(\pi_t(z_{1:t-1}),z_t)\right]
	\end{align*}
	Let us now define the ``selector function'' $\chi:\mathcal{Z} \times \mathcal{Z} \times \{\pm 1\}\mapsto \mathcal{Z}$ by
	$$
	\chi(z, z', \epsilon) = \left\{ \begin{array}{ll}
	z'  & \textrm{if } \epsilon = -1\\
	z & \textrm{if } \epsilon = 1
	\end{array}
	\right.
	$$
	In other words, $\chi_t$ selects between $z_t$ and $z'_t$ depending on the sign of $\epsilon$.
	We will use the shorthand $\chi_t(\epsilon_t)\deq \chi(z_t, z'_t, \epsilon_t)$ and $\chi_{1:t}(\epsilon_{1:t}) \deq (\chi(z_1,z'_1,\epsilon_1),\ldots, \chi(z_t,z'_t,\epsilon_t))$. We can then re-write the last statement as
	\begin{align*}
	&\multiminimax{\sup_{p_t\in \PD_t(\chi_{1:t-1}(\epsilon_{1:t-1}))}\Ex_{z_t,z'_t}\Ex_{\epsilon_t}}_{t=1}^T \sup_{\pi \in \Pi} \left[ \sum_{t=1}^T  \epsilon_t \left(\loss(\pi_t(\chi_{1:t-1}(\epsilon_{1:t-1})),\chi_t(-\epsilon_t))- \loss(\pi_t(\chi_{1:t-1}(\epsilon_{1:t-1})),\chi_t(\epsilon_t))\right)\right]
	\end{align*}
	One can indeed verify that we simply used $\chi_t$ to switch between $z_t$ and $z_t'$ according to $\epsilon_t$. Now, we can replace the second argument of the loss in both terms by a larger value to obtain an upper bound
	\begin{align*}
	&\multiminimax{\sup_{p_t\in \PD_t(\chi_{1:t-1}(\epsilon_{1:t-1}))}\Ex_{z_t,z'_t}\sup_{z''_t,z'''_t}\Ex_{\epsilon_t}}_{t=1}^T \sup_{\pi \in \Pi} \left[ \sum_{t=1}^T  \epsilon_t \left(\loss(\pi_t(\chi_{1:t-1}(\epsilon_{1:t-1})),z''_t)- \loss(\pi_t(\chi_{1:t-1}(\epsilon_{1:t-1})),z'''_t)\right)\right] \\
	&\leq 2\multiminimax{\sup_{p_t\in \PD_t(\chi_{1:t-1}(\epsilon_{1:t-1}))}\Ex_{z_t,z'_t}\sup_{z''_t}\Ex_{\epsilon_t}}_{t=1}^T \sup_{\pi \in \Pi} \left[ \sum_{t=1}^T  \epsilon_t \loss(\pi_t(\chi_{1:t-1}(\epsilon_{1:t-1})),z''_t) \right]
	\end{align*}
	since the two terms obtained by splitting the suprema are the same. We now pass to the suprema over $z_t,z'_t$, noting that the constraints need to hold:
	\begin{align*}
	&2\multiminimax{\sup_{z_t,z'_t\in {\mathcal C}_t(\chi_{1:t-1}(\epsilon_{1:t-1}))}\sup_{z''_t}\Ex_{\epsilon_t}}_{t=1}^T \sup_{\pi \in \Pi} \left[ \sum_{t=1}^T  \epsilon_t \loss(\pi_t(\chi_{1:t-1}(\epsilon_{1:t-1})),z''_t) \right] \\
	&=2\sup_{(\z,\z')\in \mathcal{C'}}\sup_{\z''} \Ex_\epsilon \sup_{\pi\in\Pi} \left[ \sum_{t=1}^T
		\epsilon_t \pi_t(\chi(\z_1,\z'_1,\epsilon_1),\ldots,\chi(\z_{t-1}(\epsilon), \z'_{t-1}(\epsilon), \epsilon_{t-1}) ), \z''(\epsilon)\right] = (*)
	\end{align*}
	where in the last step we passed to the tree notation. Importantly, the pair $(\z,\z')$ of trees does not range over all pairs, but only over those which satisfy the constraints:
	$${\mathcal C}' = \Big\{(\z,\z'): \forall \epsilon\in\{\pm1\}^T, \forall t\in[T], ~~ \z_t(\epsilon),\z'_t(\epsilon)\in {\mathcal C}_t(\chi(\z_1,\z'_1,\epsilon_1), \ldots, \chi(\z_{t-1}(\epsilon), \z'_{t-1}(\epsilon), \epsilon_{t-1})) \Big\}$$

	Now, given the pair $(\z,\z')\in{\mathcal C'}$, define a $\Z$-valued tree of depth $T$ as:
	$$\tilde{\w}_1 = \emptyset,~~~~ \tilde{\w}_t(\epsilon) = \chi (\z_{t-1}(\epsilon),\z'_{t-1}(\epsilon), \epsilon_{t-1}) ~~\mbox{for all}~~ t>1$$
	Clearly, this is a well-defined tree, and we now claim that it satisfies the constraints along every path. Indeed, we need to check that for any $\epsilon$ and $t$, both $\tilde{\w}_t(\epsilon_{1:t-2},+1),\tilde{\w}_t(\epsilon_{1:t-2},-1) \in {\mathcal C}_t (\tilde{\w}_1, \ldots, \tilde{\w}_{t-1}(\epsilon_{1:t-2}))$. This amounts to checking, by definition of $\tilde{\w}$ and the selector $\chi$, that
	$$\z_{t-1}(\epsilon_{1:t-2}), \z'_{t-1}(\epsilon_{1:t-2}) \in {\mathcal C}_{t-1}(\chi(\z_1,\z'_1,\epsilon_1), \ldots,\chi(\z_{t-2}(\epsilon), \z'_{t-2}(\epsilon), \epsilon_{t-2}) ) \ .$$
	But this is true because $(\z,\z')\in{\mathcal C'}$. Hence, $\tilde{\w}$ constructed from $\z,\z'$ satisfies the constraints along every path.
	
	We can therefore upper bound the expression in $(*)$ by twice
	$$\sup_{\tilde{\w}\in{\mathbf C}}\sup_{\z''} \Ex_\epsilon \sup_{\pi\in\Pi} \left[ \sum_{t=1}^T
		\epsilon_t \ell(\pi_t(\tilde{\w}_1(\epsilon),\ldots,\tilde{\w}_{t-1}(\epsilon)), \z''(\epsilon))\right].$$
	Define $\w^{*} = \tilde{\w}(-1)$ and $\w^{**} = \tilde{\w}(+1)$, we can expend the expectation with respect to $\epsilon_1$ of the above expression by
	\begin{align*}
	&\frac{1}{2}\sup_{\w^*\in{\mathbf C}}\sup_{\z''} \Ex_{\epsilon_{2:T}} \sup_{\pi\in\Pi} \left[ - \ell(\pi_1(\cdot),\z_1''(\cdot))+\sum_{t=2}^T \epsilon_t \ell(\pi_t(\w^{*}(\epsilon)), \z''(\epsilon))\right] \\
	&+ \frac{1}{2}\sup_{\w^{**}\in{\mathbf C}}\sup_{\z''} \Ex_{\epsilon_{2:T}} \sup_{\pi\in\Pi} \left[ \ell(\pi_1(\cdot),\z_1''(\cdot))+ \sum_{t=2}^T \epsilon_t \ell(\pi_t(\w^{**}(\epsilon)), \z''(\epsilon))\right].
	\end{align*}
	With the assumption that we do not suffer lose at the first round, which means $\ell(\pi_1(\cdot),\z_1''(\cdot))=0$, we can see that both terms achieve the suprema with the same $\w^{*} = \w^{**}$. Therefore, the above expression can be rewrite as
	\begin{align*}
	\sup_{\w \in{\mathbf C}}\sup_{\z''} \Ex_{\epsilon_{2:T}} \sup_{\pi\in\Pi} \left[\sum_{t=1}^T \epsilon_t \ell(\pi_t(\w(\epsilon)), \z''(\epsilon))\right]
	\end{align*}	
	which is precisely \eqref{eq:constained_rademacher_bound}. This concludes the proof of Theorem~\ref{thm:compete_with_strategies}.

\end{proof}

\begin{proof}[\textbf{Proof of Theorem~\ref{thm:main_lip_contraction}}]
By convexity of the loss,
\begin{align*}
	&\multiminimax{\sup_{x_t\in\X}\inf_{q_t \in \Delta(\Y)} \sup_{y_t \in \Y} \Ex_{\hat{y}_t \sim q_t} }_{t=1}^T \left[ \sum_{t=1}^T \ell(\hat{y}_t,y_t)-\inf_{\pi \in \Pi} \sum_{t=1}^T \ell(\pi_t(x_{1:t}, y_{1:t-1}),y_t)\right]\\
	&\leq\multiminimax{\sup_{x_t\in\X}\inf_{q_t \in \Delta(\Y)} \sup_{y_t \in \Y} \Ex_{\hat{y}_t \sim q_t} }_{t=1}^T \sup_{\pi \in \Pi}\left[  \sum_{t=1}^T \ell'(\hat{y}_t,y_t) (\hat{y}_t-\pi_t(x_{1:t}, y_{1:t-1}))\right]\\
	&\leq\multiminimax{\sup_{x_t\in\X}\inf_{q_t \in \Delta(\Y)} \sup_{y_t \in \Y} \Ex_{\hat{y}_t \sim q_t} \sup_{s_t\in[-L,L]}}_{t=1}^T \sup_{\pi \in \Pi}\left[  \sum_{t=1}^T s_t (\hat{y}_t-\pi_t(x_{1:t}, y_{1:t-1}))\right]
\end{align*}
where in the last step we passed to an upper bound by allowing for the worst-case choice $s_t$ of the derivative. We will often omit the range of the variables in our notation, and it is understood that $s_t$'s range over $[-L,L]$, while $y_t,\hat{y}_t$ over $\Y$ and $x_t$'s over $\X$. Now, by Jensen's inequality, we pass to an upper bound by exchanging $\Ex_{\hat{y}_t}$ and $\sup_{y_t \in \Y}$:
\begin{align*}
	&\multiminimax{\sup_{x_t}\inf_{q_t \in \Delta(\Y)} \Ex_{\hat{y}_t \sim q_t} \sup_{y_t} \sup_{s_t}}_{t=1}^T \sup_{\pi \in \Pi}\left[  \sum_{t=1}^T s_t (\hat{y}_t-\pi_t(x_{1:t}, y_{1:t-1}))\right]\\
	&=\multiminimax{\sup_{x_t}\inf_{\hat{y}_t \in \Y} \sup_{y_t,s_t}}_{t=1}^T \sup_{\pi \in \Pi}\left[  \sum_{t=1}^T s_t (\hat{y}_t-\pi_t(x_{1:t}, y_{1:t-1}))\right]
\end{align*}
Consider the last step, assuming all the other variables fixed:
\begin{align*}
	&\sup_{x_T}\inf_{\hat{y}_T}\sup_{y_T,s_T} \sup_{\pi \in \Pi}\left[  \sum_{t=1}^T s_t (\hat{y}_t-\pi_t(x_{1:t}, y_{1:t-1}))\right] \\
	&=\sup_{x_T}\inf_{\hat{y}_T}\sup_{p_T \in \Delta(\Y \times [-L,L])}\Ex_{(y_T,s_T)\sim p_T} \sup_{\pi \in \Pi}\left[  \sum_{t=1}^T s_t (\hat{y}_t-\pi_t(x_{1:t}, y_{1:t-1}))\right]
\end{align*}
where the distribution $p_T$ ranges over all distributions on $\Y \times [-L,L]$. Now observe that the function inside the infimum is convex in $\hat{y}_T$, and the function inside $\sup_{p_T}$ is linear in the distribution $p_T$. Hence, we can appeal to the minimax theorem, obtaining equality of the last expression to
\begin{align*}
	&\sup_{x_T}\sup_{p_T \in \Delta(\Y \times [-L,L])}\inf_{\hat{y}_T}\Ex_{(y_T,s_T)\sim p_T} \left[  \sum_{t=1}^T s_t \hat{y}_t-\inf_{\pi \in \Pi} \sum_{t=1}^T s_t\pi_t(x_{1:t}, y_{1:t-1}))\right] \\
	&=\sum_{t=1}^{T-1} s_t \hat{y}_t + \sup_{x_T}\sup_{p_T}\inf_{\hat{y}_T}\Ex_{(y_T,s_T) \sim p_T} \left[  s_T \hat{y}_T-\inf_{\pi \in \Pi} \sum_{t=1}^T s_t\pi_t(x_{1:t}, y_{1:t-1}))\right] \\
	&=\sum_{t=1}^{T-1} s_t \hat{y}_t + \sup_{x_T}\sup_{p_T} \left[  \inf_{\hat{y}_T}\left(\Ex_{(y_T,s_T)\sim p_T} s_T\right) \hat{y}_T- \Ex_{(y_T,s_T)\sim p_T} \inf_{\pi \in \Pi} \sum_{t=1}^T s_t\pi_t(x_{1:t}, y_{1:t-1}))\right] \\
	&= \sum_{t=1}^{T-1} s_t \hat{y}_t + \sup_{x_T}\sup_{p_T} \Ex_{(y_T,s_T)\sim p_T} \left[  \inf_{\hat{y}_T}\left(\Ex_{(y_T,s_T)\sim p_T} s_T\right) \hat{y}_T- \inf_{\pi \in \Pi} \sum_{t=1}^T s_t\pi_t(x_{1:t}, y_{1:t-1}))\right] %\\
%	&= \sum_{t=1}^{T-1} s_t \hat{y}_t + \sup_{x_T}\sup_{y_T}\sup_{p_T}\Ex_{s_T\sim p_T} \left[  \inf_{\hat{y}_T}\left(\Ex_{s_T\sim p_T} s_T\right) \hat{y}_T- \inf_{\pi \in \Pi} \sum_{t=1}^T s_t\pi_t(x_{1:t}, y_{1:t-1}))\right]
\end{align*}
%where the last equality holds because the expression is linear in $p'_T$ (unlike $p_T$). 
We can now upper bound the choice of $\hat{y}_T$ by that given by $\pi_T$, yielding an upper bound
\begin{align*}
&\sum_{t=1}^{T-1} s_t \hat{y}_t + \sup_{x_T,p_T}\Ex_{(y_T,s_T)\sim p_T} \sup_{\pi \in \Pi} \left[  \inf_{\hat{y}_T}\left(\Ex_{(y_T,s_T)\sim p_T} s_T\right) \hat{y}_T- \sum_{t=1}^T s_t\pi_t(x_{1:t}, y_{1:t-1}))\right] \\
&=\sum_{t=1}^{T-1} s_t \hat{y}_t + \sup_{x_T,p_T}\Ex_{(y_T,s_T)\sim p_T} \sup_{\pi \in \Pi} \left[  \left(\Ex_{(y'_T,s'_T)\sim p_T} s'_T - s_T\right) \pi_T(x_{1:T},y_{1:T-1})- \sum_{t=1}^{T-1} s_t\pi_t(x_{1:t}, y_{1:t-1}))\right]
\end{align*}
It is not difficult to verify that this process can be repeated for $T-1$ and so on. The resulting upper bound is therefore
\begin{align*}
	\mathcal{V}^{S}_T(\Pi) &\leq \multiminimax{\sup_{x_t, p_t}\Ex_{(y_t,s_t)\sim p_t}}_{t=1}^T \sup_{\pi \in \Pi}\left[  \sum_{t=1}^T \left(\Ex_{(y'_t,s'_t)\sim p_t} s'_t - s_t\right)\pi_t(x_{1:t}, y_{1:t-1})\right] \\
	&\leq \multiminimax{\sup_{x_t, p_t}\Ex_{\underset{(y'_t,s'_t) \sim p_t}{(y_t,s_t) \sim p_t}}}_{t=1}^T \sup_{\pi \in \Pi}\left[  \sum_{t=1}^T \left(s'_t - s_t\right)\pi_t(x_{1:t}, y_{1:t-1})\right] \\
	&= \multiminimax{\sup_{x_t, p_t}\Ex_{\underset{(y'_t,s'_t) \sim p_t}{(y_t,s_t) \sim p_t}}\Ex_{\epsilon_t}}_{t=1}^T \sup_{\pi \in \Pi}\left[  \sum_{t=1}^T \epsilon_t\left(s'_t - s_t\right)\pi_t(x_{1:t}, y_{1:t-1})\right] \\
	&\leq \multiminimax{\sup_{x_t}\sup_{\underset{(y'_t,s'_t)}{(y_t,s_t)}}\Ex_{\epsilon_t}}_{t=1}^T \sup_{\pi \in \Pi}\left[  \sum_{t=1}^T \epsilon_t\left(s'_t - s_t\right)\pi_t(x_{1:t}, y_{1:t-1})\right] \\
		&\leq \multiminimax{\sup_{x_t , y_t}\sup_{s'_t,s_t}\Ex_{\epsilon_t}}_{t=1}^T \sup_{\pi \in \Pi}\left[  \sum_{t=1}^T \epsilon_t\left(s'_t - s_t\right)\pi_t(x_{1:t}, y_{1:t-1})\right] \\
	&\leq 2 \multiminimax{\sup_{x_t,y_t}\sup_{s_t\in[-L,L]}\Ex_{\epsilon_t}}_{t=1}^T \sup_{\pi \in \Pi}\left[  \sum_{t=1}^T \epsilon_t s_t\pi_t(x_{1:t}, y_{1:t-1})\right]
\end{align*}
Since the expression is convex in each $s_t$, we can replace the range of $s_t$ by $\{-L,L\}$, or, equivalently,
\begin{align}
	\label{eq:inter1app}
	\mathcal{V}^{S}_T(\Pi)  \leq 2L \multiminimax{\sup_{x_t,y_t}\sup_{s_t\in\{-1,1\}}\Ex_{\epsilon_t}}_{t=1}^T \sup_{\pi \in \Pi}\left[  \sum_{t=1}^T \epsilon_t s_t\pi_t(x_{1:t}, y_{1:t-1})\right]
\end{align}
Now consider any arbitrary function $\psi : \{\pm 1\} \mapsto \reals$, we have that
$$
\sup_{s \in \{\pm 1\}} \Es{\epsilon}{\psi(s\cdot\epsilon)} = \sup_{s \in \{\pm 1\}} \frac{1}{2}\left(\psi(+s) + \psi(-s)\right) =  \frac{1}{2} \left(\psi(+1) + \psi(-1)\right) = \Es{\epsilon}{\psi(\epsilon)}
$$
Since in Equation~\eqref{eq:inter1app}, for each $t$, $s_t$ and $\epsilon_t$ appear together as $\epsilon_t \cdot s_t$ using the above equation repeatedly, we conclude that
\begin{align*}
	\mathcal{V}^{S}_T(\Pi)  \leq 2L \multiminimax{\sup_{x_t,y_t}\Ex_{\epsilon_t}}_{t=1}^T \sup_{\pi \in \Pi}\left[  \sum_{t=1}^T \epsilon_t \pi_t(x_{1:t}, y_{1:t-1})\right] = \sup_{\x,\y} \Ex_{\epsilon} \sup_{\pi\in\Pi}\left[ \sum_{t=1}^T \epsilon_t \pi_t(\x_{1:t}(\epsilon), \y_{1:t-1}(\epsilon))\right]
\end{align*}

The lower bound is obtained by the same argument as in \cite{RakSriTew10nips}.

\end{proof}

\begin{proof}[\textbf{Proof of Theorem~\ref{theorem:admissible}}]
	Denote $L_t(\pi) = \sum_{s=1}^t \ell(\pi_s(z_{1:s-1}),z_s)$. The first step of the proof is an application of the minimax theorem (we assume the necessary conditions hold):
	\begin{align*}
		&\inf_{q_t \in \Delta(\F)} \sup_{z_t \in \Z} \left\{ \Eunderone{f_t \sim q_t}\left[\ell(f_t,z_t)\right] + \sup_{\z, \w} \Ex_{\epsilon_{t+1:T}} \sup_{\pi \in\Pi} \left[ 2\sum_{s=t+1}^T \epsilon_s\ell(\pi_s((z_{1:t},\w_{1:s-t-1}(\epsilon)),\z_{s-t}(\epsilon)) - L_t(\pi) \right]\right\} \\
		&=\sup_{p_t \in \Delta(\Z)} \inf_{f_t \in \F} \left\{ \Eunderone{z_t \sim p_t}\left[\ell(f_t,z_t)\right] + \Eunderone{z_t \sim p_t}\sup_{\z,\w} \Ex_{\epsilon_{t+1:T}} \sup_{\pi \in \Pi} \left[ 2\sum_{s=t+1}^T \epsilon_s\ell(\pi_s((z_{1:t},\w_{1:s-t-1}(\epsilon)),\z_{s-t}(\epsilon)) - L_t(\pi) \right]\right\}
	\end{align*}
	For any $p_t\in \Delta(\Z)$, the infimum over $f_t$ of the above expression is equal to
	\begin{align*}
		&\Ex_{z_t \sim p_t}\sup_{\z,\w} \Ex_{\epsilon_{t+1:T}} \sup_{\pi\in\Pi} \left[ 2\sum_{s=t+1}^T \epsilon_s\ell(\pi_s((z_{1:t},\w_{1:s-t-1}(\epsilon)),\z_{s-t}(\epsilon)) - L_{t-1}(\pi) \right.\\
		&\left.\hspace{3.5in}+ \inf_{f_t \in \F}\Eunderone{z_t \sim p_t}\left[\ell(f_t,z_t)\right] - \ell(\pi_t(z_{1:t-1}),z_t) \right] \\
		&\leq  \Ex_{z_t \sim p_t}\sup_{\z,\w} \Ex_{\epsilon_{t+1:T}} \sup_{\pi\in\Pi} \left[ 2\sum_{s=t+1}^T \epsilon_s\ell(\pi_s((z_{1:t},\w_{1:s-t-1}(\epsilon)),\z_{s-t}(\epsilon)) - L_{t-1}(\pi) \right.\\
		&\left.\hspace{3.5in}+ \Ex_{z_t \sim p_t}\left[\ell(\pi_t(z_{1:t-1}),z_t)\right] - \ell(\pi_t(z_{1:t-1}),z_t) \right]  \\
		&\leq  \Ex_{z_t,z'_t \sim p_t} \sup_{\z,\w} \Ex_{\epsilon_{t+1:T}} \sup_{\pi\in\Pi} \left[ 2\sum_{s=t+1}^T  \epsilon_s\ell(\pi_s((z_{1:t},\w_{1:s-t-1}(\epsilon)),\z_{s-t}(\epsilon))  - L_{t-1}(\pi) \right.\\
		&\left.\hspace{3.5in}+ \ell(\pi_t(z_{1:t-1}),z'_t) - \ell(\pi_t(z_{1:t-1}),z_t) \right]
	\end{align*}
	We now argue that the independent $z_t$ and $z'_t$ have the same distribution $p_t$, and thus we can introduce a random sign $\epsilon_t$. The above expression then equals to
	\begin{align*}
	&\Ex_{z_t,z'_t \sim p_t} \Ex_{\epsilon_t} \sup_{\z,\w} \Ex_{\epsilon_{t+1:T}} \sup_{\pi\in\Pi} \left[ 2\sum_{s=t+1}^T  \epsilon_s\ell(\pi_s((z_{1:t-1},\chi_t(\epsilon_t), \w_{1:s-t-1}(\epsilon)),\z_{s-t}(\epsilon))  - L_{t-1}(\pi) \right.\\
	&\left.\hspace{2in}+ \epsilon_t ( \ell(\pi_t(z_{1:t-1}),\chi_t(-\epsilon_t))) - \ell(\pi_t(z_{1:t-1}),\chi_t(\epsilon_t))) \right] \\
	&\leq \Ex_{z_t,z'_t \sim p_t} \sup_{z'',z'''} \Ex_{\epsilon_t} \sup_{\z,\w} \Ex_{\epsilon_{t+1:T}} \sup_{\pi\in\Pi}  \left[ 2\sum_{s=t+1}^T  \epsilon_s\ell(\pi_s((z_{1:t-1},\chi_t(\epsilon_t), \w_{1:s-t-1}(\epsilon)),\z_{s-t}(\epsilon))  - L_{t-1}(\pi) \right.\\
	&\left.\hspace{2in}+ \epsilon_t ( \ell(\pi_t(z_{1:t-1}),z''_t) - \ell(\pi_t(z_{1:t-1}),z'''_t)) \right] \\
	\end{align*}
	Splitting the resulting expression into two parts, we arrive at the upper bound of
	\begin{align*}
	& 2 \Ex_{z_t,z'_t \sim p_t} \sup_{z''} \Ex_{\epsilon_t} \sup_{\z,\w} \Ex_{\epsilon_{t+1:T}} \sup_{\pi\in\Pi}  \left[\sum_{s=t+1}^T  \epsilon_s\ell(\pi_s((z_{1:t-1},\chi_t(\epsilon_t), \w_{1:s-t-1}(\epsilon)),\z_{s-t}(\epsilon))  - \frac{1}{2} L_{t-1}(\pi) + \epsilon_t \ell(\pi_t(z_{1:t-1}),z''_t) \right]  \\
	&\leq \sup_{z,z',z''} \Ex_{\epsilon_t} \sup_{\z,\w} \Ex_{\epsilon_{t+1:T}} \sup_{\pi\in\Pi}  \left[\sum_{s=t+1}^T  2 \epsilon_s\ell(\pi_s((z_{1:t-1},\chi_t(\epsilon_t), \w_{1:s-t-1}(\epsilon)),\z_{s-t}(\epsilon)) - L_{t-1}(\pi) + \epsilon_t \ell(\pi_t(z_{1:t-1}),z''_t) \right] \\
	& \leq \Rad_T (\Pi | z_1,\ldots,z_{t-1}).
	\end{align*}
	The first inequality is true as  we upper bounded the expectation by the supremum. The last inequality is easy to verify, as we are effectively filling in a root $z_t$ and $z'_t$ for the two subtrees, for $\epsilon_t=+1$ and $\epsilon_t=-1$, respectively, and jointing the two trees with a $\emptyset$ root.
	
	One can see that the proof of admissibility corresponds to one step minimax swap and symmetrization in the proof of \cite{RakSriTew10nips}. In contrast, in the latter paper, all $T$ minimax swaps are performed at once, followed by $T$ symmetrization steps.
\end{proof}

\begin{proof}[\textbf{Proof of Lemma~\ref{lem:admissibility_of_lin_l1_linf}}]
        The first step of the proof is an application of the minimax theorem (we assume the necessary conditions hold):
	\begin{align*}
		\inf_{q_t \in \Delta(\F)} \sup_{z_t \in \Z} \left\{ \Ex_{f_t\sim q_t} f_t \cdot z_t  + \Ex_{\epsilon_{t+1:T}} \max_{1 \leq s \leq T} \left| \sum_{i=s}^T a^t_i(\epsilon) \right| \right\} =\sup_{p_t \in \Delta(\Z)} \inf_{f_t \in \F} \left\{ f_t \cdot \Ex_{z_t \sim p_t} z_t + \Eunderone{z_t \sim p_t} \Ex_{\epsilon_{t+1:T}} \max_{1 \leq s \leq T} \left| \sum_{i=s}^T a^t_i(\epsilon) \right| \right\}
	\end{align*}
	For any $p_t\in \Delta(\Z)$, the infimum over $f_t$ of the above expression is equal to
	\begin{align*}
		&- \left\vert \Eunderone{z_t \sim p_t} z_t \right\vert + \Eunderone{z_t \sim p_t} \Ex_{\epsilon_{t+1:T}} \max\left\{\max_{s > t} \left| \sum_{i=s}^T a^t_i(\epsilon) \right|, \max_{s \leq t} \left| \sum_{i=s}^T a^t_i(\epsilon) \right|\right\} \\
		&\leq  \Eunderone{z_t \sim p_t} \Ex_{\epsilon_{t+1:T}} \max\left\{\max_{s > t} \left| \sum_{i=s}^T a^t_i(\epsilon) \right|, \max_{s \leq t} \left| \sum_{i=s}^T a^t_i(\epsilon) + \Eunderone{z'_t \sim p_t} z'_t  \right|\right\} \\
		&\leq  \Eunderone{z_t,z'_t \sim p_t} \Ex_{\epsilon_{t+1:T}} \max\left\{\max_{s > t} \left| \sum_{i=s}^T a^t_i(\epsilon) \right|, \max_{s \leq t} \left| \sum_{i \geq s,  i \neq t} a^t_i(\epsilon) +(z'_t- z_t)  \right|\right\}
	\end{align*}
	We now argue that the independent $z_t$ and $z'_t$ have the same distribution $p_t$, and thus we can introduce a random sign $\epsilon_t$. The above expression then equals to
	\begin{align*}
		&\Eunderone{z_t,z'_t \sim p_t} \Ex_{\epsilon_{t:T}} \max\left\{\max_{s > t} \left| \sum_{i=s}^T a^t_i(\epsilon) \right|, \max_{s \leq t} \left| \sum_{i \geq s, i \neq t} a^t_i(\epsilon) +\epsilon_t (z'_t-z_t)  \right|\right\} \\
		&\leq \Eunderone{z_t \sim p_t} \Ex_{\epsilon_{t:T}} \max\left\{\max_{s > t} \left| \sum_{i=s}^T a^t_i(\epsilon) \right|, \max_{s \leq t} \left| \sum_{i \geq s, i \neq t} a^t_i(\epsilon) + 2 \epsilon_t z_t  \right|\right\}
	\end{align*}
	Now, the supremum over $p_t$ is achieved at a delta distribution, yielding an upper bound
	\begin{align*}
		& \sup_{z_t \in [-1,1]} \Ex_{\epsilon_{t:T}} \max\left\{\max_{s > t} \left| \sum_{i=s}^T a^t_i(\epsilon) \right|, \max_{s \leq t} \left| \sum_{i \geq s, i \neq t} a^t_i(\epsilon) + 2 \epsilon_t z_t  \right|\right\} \\
		&\leq \Ex_{\epsilon_{t:T}} \max\left\{\max_{s > t} \left| \sum_{i=s}^T a^t_i(\epsilon) \right|, \max_{s \leq t} \left| \sum_{i \geq s, i \neq t} a^t_i(\epsilon) + 2 \epsilon_t  \right|\right\} \\
		&=\Ex_{\epsilon_{t:T}} \max_{1 \leq s \leq T} \left| \sum_{i=s}^T a^{t-1}_i(\epsilon) \right|
	\end{align*}
\end{proof}

\begin{proof}[\textbf{Proof of Lemma~\ref{lem:alpha_beta_relaxation_admissible}}]
	Denote $$L_t(\alpha,\beta) = \sum_{s=1}^t \left| \frac{\frac{\sum_{i=1}^{s-1} z_i}{s+\alpha-2}}{1+\frac{\beta-1}{s+\alpha-2}} -z_s \right|.$$ The first step of the proof is an application of the minimax theorem:
	\begin{align*}
		&\inf_{q_t \in \Delta(\F)} \sup_{z_t \in \Z} \left\{ \Eunderone{f_t \sim q_t}\left|f_t-z_t\right| + \Ex_{\epsilon_{t+1:T}} \sup_{\alpha,\beta} \left[  2 \sum_{s=t+1}^T\epsilon_s \cdot \frac{1}{1+\frac{\beta-1}{s+\alpha-2}} - L_t(\alpha,\beta) \right]\right\} \\
		&=\sup_{p_t \in \Delta(\Z)} \inf_{f_t \in \F} \left\{ \Eunderone{z_t \sim p_t}\left|f_t-z_t\right| + \Eunderone{z_t \sim p_t} \Ex_{\epsilon_{t+1:T}} \sup_{\alpha,\beta} \left[ 2 \sum_{s=t+1}^T\epsilon_s \cdot \frac{1}{1+\frac{\beta-1}{s+\alpha-2}} - L_t(\alpha,\beta) \right]\right\}
	\end{align*}
	For any $p_t\in \Delta(\Z)$, the infimum over $f_t$ of the above expression is equal to
	\begin{align*}
		&\Eunderone{z_t \sim p_t} \En_{\epsilon_{t+1:T}} \sup_{\alpha,\beta} \left[ 2 \sum_{s=t+1}^T\epsilon_s \cdot \frac{1}{1+\frac{\beta-1}{s+\alpha-2}}  - L_{t-1}(\alpha,\beta) + \inf_{f_t \in \F}\Eunderone{z_t \sim p_t}\left| f_t - z_t\right| - \left| \frac{\frac{\sum_{i=1}^{t-1} z_i}{t+\alpha-2}}{1+\frac{\beta-1}{t+\alpha-2}} -z_t \right| \right] \\
		&\leq  \Eunderone{z_t \sim p_t} \En_{\epsilon_{t+1:T}} \sup_{\alpha,\beta} \left[ 2 \sum_{s=t+1}^T\epsilon_s \cdot \frac{1}{1+\frac{\beta-1}{s+\alpha-2}}  - L_{t-1}(\alpha,\beta) + \Eunderone{z'_t \sim p_t} \left| \frac{\frac{\sum_{i=1}^{t-1} z_i}{t+\alpha-2}}{1+\frac{\beta-1}{t+\alpha-2}} -z'_t \right| - \left| \frac{\frac{\sum_{i=1}^{t-1} z_i}{t+\alpha-2}}{1+\frac{\beta-1}{t+\alpha-2}} -z_t \right| \right]  \\
		&\leq  \Eunderone{z_t,z'_t \sim p_t} \En_{\epsilon_{t+1:T}} \sup_{\alpha,\beta} \left[ 2 \sum_{s=t+1}^T\epsilon_s \cdot \frac{1}{1+\frac{\beta-1}{s+\alpha-2}} - L_{t-1}(\alpha,\beta) + \left| \frac{\frac{\sum_{i=1}^{t-1} z_i}{t+\alpha-2}}{1+\frac{\beta-1}{t+\alpha-2}} -z'_t \right| - \left| \frac{\frac{\sum_{i=1}^{t-1} z_i}{t+\alpha-2}}{1+\frac{\beta-1}{t+\alpha-2}} -z_t \right| \right]
	\end{align*}
	We now argue that the independent $z_t$ and $z'_t$ have the same distribution $p_t$, and thus we can introduce a random sign $\epsilon_t$. The above expression then equals to
	\begin{align*}
	& \Eunderone{z_t,z'_t \sim p_t} \mathbb{E}_{\epsilon_t} \En_{\epsilon_{t+1:T}} \sup_{\alpha,\beta} \left[ 2 \sum_{s=t+1}^T\epsilon_s \cdot \frac{1}{1+\frac{\beta-1}{s+\alpha-2}} - L_{t-1}(\alpha,\beta) + \epsilon_t\left(\left| \frac{\frac{\sum_{i=1}^{t-1} z_i}{t+\alpha-2}}{1+\frac{\beta-1}{t+\alpha-2}} -z'_t \right| - \left| \frac{\frac{\sum_{i=1}^{t-1} z_i}{t+\alpha-2}}{1+\frac{\beta-1}{t+\alpha-2}} -z_t \right| \right) \right] \\
	&\leq \sup_{z_t,z'_t \in\Z} \En_{\epsilon_{t:T}} \sup_{\alpha,\beta} \left[ 2 \sum_{s=t+1}^T\epsilon_s \cdot \frac{1}{1+\frac{\beta-1}{s+\alpha-2}} - L_{t-1}(\alpha,\beta) + \epsilon_t\left(\left| \frac{\frac{\sum_{i=1}^{t-1} z_i}{t+\alpha-2}}{1+\frac{\beta-1}{t+\alpha-2}} -z'_t \right| - \left| \frac{\frac{\sum_{i=1}^{t-1} z_i}{t+\alpha-2}}{1+\frac{\beta-1}{t+\alpha-2}} -z_t \right| \right) \right]
	\end{align*}
	where we upper bounded the expectation by the supremum. Splitting the resulting expression into two parts, we arrive at the upper bound of
	\begin{align*}
	& 2 \sup_{z_t \in\Z} \En_{\epsilon_{t:T}} \sup_{\alpha,\beta} \left[\sum_{s=t+1}^T\epsilon_s \cdot \frac{1}{1+\frac{\beta-1}{s+\alpha-2}} - \frac{1}{2} L_{t-1}(\alpha,\beta) + \epsilon_t \left| \frac{\frac{\sum_{i=1}^{t-1} z_i}{t+\alpha-2}}{1+\frac{\beta-1}{t+\alpha-2}} -z_t \right| \right] \\
	&= 2 \sup_{z_t \in\Z} \En_{\epsilon_{t:T}} \sup_{\alpha,\beta} \left[\sum_{s=t+1}^T\epsilon_s \cdot \frac{1}{1+\frac{\beta-1}{s+\alpha-2}} - \frac{1}{2} L_{t-1}(\alpha,\beta) + \epsilon_t \cdot \frac{\frac{\sum_{i=1}^{t-1} z_i}{t+\alpha-2}}{1+\frac{\beta-1}{t+\alpha-2}} (1-2 z_t) - \epsilon_t z_t \right] \\
	&= 2 \En_{\epsilon_{t:T}} \sup_{\alpha,\beta} \left[\sum_{s=t+1}^T\epsilon_s \cdot \frac{1}{1+\frac{\beta-1}{s+\alpha-2}} - \frac{1}{2} L_{t-1}(\alpha,\beta) + \epsilon_t \cdot \frac{\frac{\sum_{i=1}^{t-1} z_i}{t+\alpha-2}}{1+\frac{\beta-1}{t+\alpha-2}} \right]
	\end{align*}
	where the last step is due to the fact that for any $z_t\in\{0,1\}$, $\epsilon_t(1-2z_t)$ has the same distribution as $\epsilon_t$. We then proceed to upper bound
	\begin{align*}
	&2 \sup_{p} \mathbb{E}_{a \sim p} \En_{\epsilon_{t:T}} \sup_{\alpha,\beta} \left[\sum_{s=t+1}^T\epsilon_s \cdot \frac{1}{1+\frac{\beta-1}{s+\alpha-2}} - \frac{1}{2} L_{t-1}(\alpha,\beta) + \epsilon_t \cdot \frac{a}{1+\frac{\beta-1}{t+\alpha-2}} \right] \\
	&\leq 2 \sup_{a \in \{\pm 1\}} \En_{\epsilon_{t:T}} \sup_{\alpha,\beta} \left[\sum_{s=t+1}^T\epsilon_s \cdot \frac{1}{1+\frac{\beta-1}{s+\alpha-2}} - \frac{1}{2} L_{t-1}(\alpha,\beta) + \epsilon_t \cdot \frac{a}{1+\frac{\beta-1}{t+\alpha-2}} \right] 	\\
	&\leq 2 \En_{\epsilon_{t:T}} \sup_{\alpha,\beta} \left[\sum_{s=t}^T\epsilon_s \cdot \frac{1}{1+\frac{\beta-1}{s+\alpha-2}} - \frac{1}{2} L_{t-1}(\alpha,\beta) \right]
	\end{align*}
	The initial condition is trivially satisfied as
	$${\bf Rel}(z_{1:T}) = -\inf_{\alpha,\beta} \sum_{s=1}^T \left| \frac{\frac{\sum_{i=1}^{s-1} z_i}{s+\alpha-2}}{1+\frac{\beta-1}{s+\alpha-2}} -z_s \right| $$
\end{proof}

\begin{theorem} \label{theorem:gaussian_admissible} The conditional sequential Rademacher complexity with respect to $\Pi$
$$ \mathcal{G}_T (\ell,\Pi | z_1,\ldots,z_{t}) \deq \sup_{\z,\w} \Ex_{\sigma_{t+1:T}} \sup_{\pi \in\Pi} \left[ \sqrt{2\pi} \sum_{s=t+1}^T \sigma_s \ell(\pi_s((z_{1:t},\w_{1:s-t-1}(\epsilon)),\z_{s-t}(\epsilon)) -  \sum_{s=1}^t \ell(\pi_s(z_{1:s-1}),z_s) \right]$$
is admissible.
\end{theorem}

\begin{proof}[\textbf{Proof of Theorem~\ref{theorem:gaussian_admissible}}]
	Denote $L_t(\pi) = \sum_{s=1}^t \ell(\pi_s(z_{1:s-1}),z_s)$. Let $c = \Ex_{\sigma} |\sigma|=\sqrt{2/\pi}$. The first step of the proof is an application of the minimax theorem (we assume the necessary conditions hold):
	\begin{align*}
		&\inf_{q_t \in \Delta(\F)} \sup_{z_t \in \Z} \left\{ \Eunderone{f_t \sim q_t}\left[\ell(f_t,z_t)\right] + \sup_{\z, \w} \Ex_{\sigma_{t+1:T}} \sup_{\pi \in\Pi} \left[ \frac{2}{c} \sum_{s=t+1}^T \sigma_s\ell(\pi_s((z_{1:t},\w_{1:s-t-1}(\epsilon)),\z_{s-t}(\epsilon)) - L_t(\pi) \right]\right\} \\
		&=\sup_{p_t \in \Delta(\Z)} \inf_{f_t \in \F} \left\{ \Eunderone{z_t \sim p_t}\left[\ell(f_t,z_t)\right] + \Eunderone{z_t \sim p_t}\sup_{\z,\w} \Ex_{\sigma_{t+1:T}} \sup_{\pi \in \Pi} \left[ \frac{2}{c}\sum_{s=t+1}^T \sigma_s\ell(\pi_s((z_{1:t},\w_{1:s-t-1}(\epsilon)),\z_{s-t}(\epsilon)) - L_t(\pi) \right]\right\}
	\end{align*}
	For any $p_t\in \Delta(\Z)$, the infimum over $f_t$ of the above expression is equal to
	\begin{align*}
		&\Ex_{z_t \sim p_t}\sup_{\z,\w} \Ex_{\sigma_{t+1:T}} \sup_{\pi\in\Pi} \left[ \frac{2}{c} \sum_{s=t+1}^T \sigma_s\ell(\pi_s((z_{1:t},\w_{1:s-t-1}(\epsilon)),\z_{s-t}(\epsilon)) - L_{t-1}(\pi) \right.\\
		&\left.\hspace{3.5in}+ \inf_{f_t \in \F}\Eunderone{z_t \sim p_t}\left[\ell(f_t,z_t)\right] - \ell(\pi_t(z_{1:t-1}),z_t) \right] \\
		&\leq  \Ex_{z_t \sim p_t}\sup_{\z,\w} \Ex_{\sigma_{t+1:T}} \sup_{\pi\in\Pi} \left[ \frac{2}{c} \sum_{s=t+1}^T \sigma_s\ell(\pi_s((z_{1:t},\w_{1:s-t-1}(\epsilon)),\z_{s-t}(\epsilon)) - L_{t-1}(\pi) \right.\\
		&\left.\hspace{3.5in}+ \Ex_{z_t \sim p_t}\left[\ell(\pi_t(z_{1:t-1}),z_t)\right] - \ell(\pi_t(z_{1:t-1}),z_t) \right]  \\
		&\leq  \Ex_{z_t,z'_t \sim p_t} \sup_{\z,\w} \Ex_{\sigma_{t+1:T}} \sup_{\pi\in\Pi} \left[ \frac{2}{c} \sum_{s=t+1}^T  \sigma_s\ell(\pi_s((z_{1:t},\w_{1:s-t-1}(\epsilon)),\z_{s-t}(\epsilon))  - L_{t-1}(\pi) \right.\\
		&\left.\hspace{3.5in}+ \ell(\pi_t(z_{1:t-1}),z'_t) - \ell(\pi_t(z_{1:t-1}),z_t) \right]
	\end{align*}
	We now argue that the independent $z_t$ and $z'_t$ have the same distribution $p_t$, and thus we can introduce a gaussian random variable $\sigma_t$ and a random sign $\epsilon_t =\sign(\sigma_t)$. The above expression then equals to
	\begin{align*}
	&\Ex_{z_t,z'_t \sim p_t} \Ex_{\sigma_t} \sup_{\z,\w} \Ex_{\sigma_{t+1:T}} \sup_{\pi\in\Pi} \left[ \frac{2}{c} \sum_{s=t+1}^T  \sigma_s\ell(\pi_s((z_{1:t-1},\chi_t(\epsilon_t), \w_{1:s-t-1}(\epsilon)),\z_{s-t}(\epsilon))  - L_{t-1}(\pi) \right.\\
	&\left.\hspace{2in}+ \epsilon_t ( \ell(\pi_t(z_{1:t-1}),\chi_t(-\epsilon_t))) - \ell(\pi_t(z_{1:t-1}),\chi_t(\epsilon_t))) \right] \\
	&\leq \Ex_{z_t,z'_t \sim p_t} \Ex_{\sigma_t} \sup_{\z,\w} \Ex_{\sigma_{t+1:T}} \sup_{\pi\in\Pi} \left[ \frac{2}{c} \sum_{s=t+1}^T  \sigma_s\ell(\pi_s((z_{1:t-1},\chi_t(\epsilon_t), \w_{1:s-t-1}(\epsilon)),\z_{s-t}(\epsilon))  - L_{t-1}(\pi) \right.\\
	&\left.\hspace{2in}+ \epsilon_t \Ex_{\sigma_t}\left| \frac{\sigma_t}{c} \right| ( \ell(\pi_t(z_{1:t-1}),\chi_t(-\epsilon_t))) - \ell(\pi_t(z_{1:t-1}),\chi_t(\epsilon_t))) \right]
	\end{align*}
	Put the expectation outside and use the fact $\epsilon_t \vert \sigma_t \vert = \sigma_t$, we get
	\begin{align*}
	&\Ex_{z_t,z'_t \sim p_t} \Ex_{\sigma_t} \sup_{\z,\w} \Ex_{\sigma_{t+1:T}} \sup_{\pi\in\Pi} \left[ \frac{2}{c} \sum_{s=t+1}^T  \sigma_s\ell(\pi_s((z_{1:t-1},\chi_t(\epsilon_t), \w_{1:s-t-1}(\epsilon)),\z_{s-t}(\epsilon))  - L_{t-1}(\pi) \right.\\
	&\left.\hspace{2in}+ \frac{\sigma_t}{c} ( \ell(\pi_t(z_{1:t-1}),\chi_t(-\epsilon_t))) - \ell(\pi_t(z_{1:t-1}),\chi_t(\epsilon_t))) \right] \\
	&\leq \Ex_{z_t,z'_t \sim p_t} \sup_{z'',z'''} \Ex_{\sigma_t} \sup_{\z,\w} \Ex_{\sigma_{t+1:T}} \sup_{\pi\in\Pi}  \left[ \frac{2}{c} \sum_{s=t+1}^T  \epsilon_s\ell(\pi_s((z_{1:t-1},\chi_t(\epsilon_t), \w_{1:s-t-1}(\epsilon)),\z_{s-t}(\epsilon))  - L_{t-1}(\pi) \right.\\
	&\left.\hspace{2in}+ \frac{\sigma_t}{c} ( \ell(\pi_t(z_{1:t-1}),z''_t) - \ell(\pi_t(z_{1:t-1}),z'''_t)) \right]
	\end{align*}
	Splitting the resulting expression into two parts, we arrive at the upper bound of
	\begin{align*}
	& 2 \Ex_{z_t,z'_t \sim p_t} \sup_{z''} \Ex_{\sigma_t} \sup_{\z,\w} \Ex_{\sigma_{t+1:T}} \sup_{\pi\in\Pi}  \left[\frac{1}{c} \sum_{s=t+1}^T  \sigma_s\ell(\pi_s((z_{1:t-1},\chi_t(\epsilon_t), \w_{1:s-t-1}(\epsilon)),\z_{s-t}(\epsilon))  - \frac{1}{2} L_{t-1}(\pi) \right.\\
	&\left.\hspace{2in}+\frac{\sigma_t}{c} \ell(\pi_t(z_{1:t-1}),z''_t) \right]  \\
	&\leq \sup_{z,z',z''} \Ex_{\sigma_t} \sup_{\z,\w} \Ex_{\sigma_{t+1:T}} \sup_{\pi\in\Pi}  \left[\frac{2}{c}\sum_{s=t+1}^T  \sigma_s\ell(\pi_s((z_{1:t-1},\chi_t(\epsilon_t), \w_{1:s-t-1}(\epsilon)),\z_{s-t}(\epsilon)) - L_{t-1}(\pi) \right.\\
	&\left.\hspace{2in}+ \frac{2 \sigma_t}{c} \ell(\pi_t(z_{1:t-1}),z''_t) \right] \\
	& \leq \mathcal{G}_T (\ell, \Pi | z_1,\ldots,z_{t-1}).
	\end{align*}
\end{proof}

\begin{proof}[\textbf{Proof of Lemma~\ref{lem:treewalk}}]
Let $q_t$ be the randomized strategy where we draw $\epsilon_{t+1},\ldots,\epsilon_T$ uniformly at random and pick
\begin{align}\label{eq:fpltree}
	q_t(\epsilon) =\argmin{q \in [-1,1]} \sup_{z_t \in \{-1,1\}}\left\{ \Ex_{f_t \sim q} f_t \cdot z_t + \max_{1 \leq s \leq T} \left| \sum_{i=s}^T a^t_i(\epsilon) \right| \right\}
\end{align}
Then, 
\begin{align*}
&\sup_{z_t \in \{-1,1\}}\left\{ \Ex_{f_t \sim q_t} f_t \cdot z_t + \En_{\epsilon_{t+1:T}} \max_{1 \leq s \leq T} \left| \sum_{i=s}^T a^t_i(\epsilon) \right| \right\}\\
&= \sup_{z_t \in \{-1,1\}}\left\{\En_{\epsilon_{t+1:T}} \Ex_{f_t \sim q_t(\epsilon)} f_t \cdot z_t + \En_{\epsilon_{t+1:T}} \max_{1 \leq s \leq T} \left| \sum_{i=s}^T a^t_i(\epsilon) \right| \right\} \\
& \le \Es{\epsilon_{t+1:T}}{ \sup_{z_t}\left\{ \Ex_{f_t \sim q_t(\epsilon)} f_t \cdot z_t + \max_{1 \leq s \leq T} \left| \sum_{i=s}^T a^t_i(\epsilon) \right| \right\}}\\
& = \Es{\epsilon_{t+1:T}}{ \inf_{q_t \in \Delta(\F)} \sup_{z_t}\left\{ \Ex_{f_t \sim q_t} f_t \cdot z_t + \max_{1 \leq s \leq T} \left| \sum_{i=s}^T a^t_i(\epsilon) \right| \right\}}
\end{align*}
where the last step is due to the way we pick our predictor $f_t(\epsilon)$ given random draw of $\epsilon$'s in Equation \eqref{eq:fpltree}. We now apply the minimax theorem, yielding the following upper bound on the term above:
\begin{align*}
& \Es{\epsilon_{t+1:T}}{ \sup_{p_t \in \Delta(\Z)} \inf_{f_t}\left\{ \Ex_{z_t \sim p_t} f_t \cdot z_t + \Ex_{z_t \sim p_t} \max_{1 \leq s \leq T} \left| \sum_{i=s}^T a^t_i(\epsilon) \right| \right\}}
\end{align*}
This expression can be re-written as
\begin{align*}
& \Es{\epsilon_{t+1:T}}{ \sup_{p_t \in \Delta(\Z)} \Ex_{z_t \sim p_t} \inf_{f_t}\left\{ \Ex_{z'_t \sim p_t} f_t \cdot z'_t + \max_{1 \leq s \leq T} \left| \sum_{i=s}^T a^t_i(\epsilon) \right| \right\}} \\
& \leq \Es{\epsilon_{t+1:T}}{ \sup_{p_t \in \Delta(\Z)} \Ex_{z_t \sim p_t} \left\{ - \left|\Ex_{z'_t \sim p_t} z'_t\right|  + \max_{1 \leq s \leq T} \left| \sum_{i=s}^T a^t_i(\epsilon) \right| \right\}} \\
& \leq \Es{\epsilon_{t+1:T}}{ \sup_{p_t \in \Delta(\Z)} \Ex_{z_t \sim p_t} \max \left\{ \max_{s \leq t} \left| \sum_{i=s}^T a^t_i(\epsilon) + \Ex_{z'_t \sim p_t} z'_t \right|, \max_{s > t} \left| \sum_{i=s}^T a^t_i(\epsilon) \right| \right\}} \\
& \leq \Es{\epsilon_{t+1:T}}{ \sup_{p_t \in \Delta(\Z)} \Ex_{z_t, z'_t \sim p_t} \max \left\{ \max_{s \leq t} \left| \sum_{i \geq s, i \neq t}^T a^t_i(\epsilon) + (z'_t-z_t) \right|, \max_{s > t} \left| \sum_{i=s}^T a^t_i(\epsilon) \right| \right\}}
\end{align*}

We now argue that the independent $z_t$ and $z'_t$ have the same distribution $p_t$, and thus we can introduce a random sign $\epsilon_t$. The above expression then equals to
\begin{align*}
& \Es{\epsilon_{t+1:T}}{ \sup_{p_t \in \Delta(\Z)} \Ex_{z_t, z'_t \sim p_t} \Ex_{\epsilon_t} \max \left\{ \max_{s \leq t} \left| \sum_{i \geq s, i \neq t}^T a^t_i(\epsilon) + \epsilon_t (z'_t-z_t) \right|, \max_{s > t} \left| \sum_{i=s}^T a^t_i(\epsilon) \right| \right\}} \\
& \leq \Ex_{\epsilon_{t+1:T}} \sup_{z_t \in \{-1,1\}} \Ex_{\epsilon_t} \max_{1 \leq s \leq T} \left| \sum_{i=s}^T a^{t-1}_i(\epsilon) \right| = \Ex_{\epsilon_{t:T}} \max_{1 \leq s \leq T} \left| \sum_{i=s}^T a^{t-1}_i(\epsilon) \right|
\end{align*}

\end{proof}

\begin{proof}[\textbf{Proof of Lemma~\ref{lem:rls}}]
	Given an $\X$-valued tree $\x$ and a $\Y$-valued tree $\y$, let us write $\bX_t(\epsilon)$ for the matrix consisting of $(\x_1(\epsilon),\ldots,\x_{t-1}(\epsilon))$ and $\bY_t$ for the vector $(\y_1(\epsilon),\ldots,\y_{t-1}(\epsilon))$.
	By Theorem~\ref{thm:main_lip_contraction}, the minimax regret is bounded by
	\begin{align*}
		&4\sup_{\x,\y}\mathbb{E}_{\epsilon} \sup_{\pi^{\lambda,w_0} \in \Pi}\left[ \sum_{t=1}^T \epsilon_t \pi^{\lambda, w_0}_t(\x_{1:t}(\epsilon),\y_{1:t-1}(\epsilon)) \right] \\
		&=  4\sup_{\x,\y}\mathbb{E}_{\epsilon} \sup_{\lambda,w_0}\left[ \sum_{t=1}^T \epsilon_t c\left(\left\langle (\bX_t(\epsilon)^\tr \bX_t(\epsilon)+ \lambda I)^{-1} \bX_t(\epsilon)^\tr \bY_t(\epsilon), \x_t(\epsilon) \right\rangle + \inner{w_0,\x_{t}(\epsilon)}\right) \right]  
	\end{align*}
	Since the output of the clipped strategies in $\Pi$ is between $-1$ and $1$, the Dudley integral gives an upper bound 
	\begin{align*}
		\Rad(\Pi,(\x,\y)) \le \inf_{\alpha\geq 0}\left\{4 \alpha T + 12\sqrt{T} \int_{\alpha}^{1} \sqrt{\log \ \cN_2(\Pi, (\x,\y), \delta ) \ } d \delta \right\}
	\end{align*}
	Define the set of strategies before clipping:
	$$\Pi' = \left\{\pi': \pi'_t(x_{1:t},y_{1:t-1})=\inner{w_0 + (X^\tr X+ \lambda I)^{-1} X^\tr Y, x_t}, \| w_0 \| \leq 1, ~ \lambda  > \lambda_{\min}\right\}$$
	If $V$ is a $\delta$-cover of $\Pi'$ on $(\x,\y)$, then $V$ is also an $\delta$-cover of $\Pi$ as $|c(x)-c(x')| \leq |x-y|$. Therefore, for any $(\x,\y)$,
	$$\mathcal{N}_2(\Pi, (\x,\y), \delta) \leq \mathcal{N}_2(\Pi', (\x,\y), \delta)$$
	and
	\begin{align*}
		\Rad(\Pi,(\x,\y)) \le \inf_{\alpha\geq 0}\left\{4 \alpha T + 12\sqrt{T} \int_{\alpha}^{1} \sqrt{\log \ \cN_2(\Pi', (\x,\y), \delta ) \ } d \delta \right\}.
	\end{align*}

	If $W$ is a $\delta/2$-cover of the set of strategies $\Pi_{w_0} = \left\{\inner{w_0, \x_t(\epsilon)} : \| w_0 \| \leq 1\right\} $ on a tree $\x$, and $\Lambda$ is a $\delta/2$-cover of the set of strategies 
	$$\Pi_{\lambda} = \left\{\pi: \pi_t(x_{1:t},y_{1:t-1}) = \inner{(X^\tr X+ \lambda I)^{-1} X^\tr Y, x_t} : \lambda  > \lambda_{\min}\right\}$$ 
	then $W \times \Lambda$ is an $\delta$-cover of $\Pi'$. Therefore,  
	$$\mathcal{N}_2(\Pi', (\x,\y), \delta) \leq \mathcal{N}_2(\Pi_{w_0}, (\x,\y), \delta/2) \times \mathcal{N}_2(\Pi_{\lambda}, (\x,\y), \delta/2).$$
	Hence,
	\begin{align*}
		\Rad(\Pi, (\x,\y)) &\le \inf_{\alpha\geq 0}\left\{4 \alpha T + 12\sqrt{T} \int_{\alpha}^{1} \sqrt{\log \ \cN_2(\Pi_{w_0}, (\x,\y), \delta/2 ) + \log \ \cN_2(\Pi_{\lambda}, (\x,\y), \delta/2 ) \ } d \delta \right\} \\
		%& \leq \inf_{\alpha\geq 0}\left\{4 \alpha T + 12\sqrt{T} \int_{\alpha}^{1} \sqrt{\log \ \cN_2(\Pi_{w_0}, \delta/2 ) \ } d \delta + 12\sqrt{T} \int_{\alpha}^{1} \sqrt{ \log \ \cN_2(\Pi_{\lambda}, \delta/2 ) \ } d \delta \right\} \\
		& \leq \inf_{\alpha\geq 0}\left\{4 \alpha T + 12\sqrt{T} \int_{\alpha}^{1} \sqrt{\log \ \cN_2(\Pi_{w_0}, (\x,\y), \delta/2 ) \ } d \delta \right\} \\
		&+ 12\sqrt{T} \int_{0}^{1} \sqrt{ \log \ \cN_2(\Pi_{\lambda}, (\x,\y), \delta/2 ) \ } d \delta
	\end{align*}
	The first term is the Dudley integral of the set of static strategies $\Pi_{w_0}$ given by $w_0\in B_2(1)$, and it is exactly the complexity studied in \cite{RakSriTew10nips} where it is shown to be $O(\sqrt{T\log (T)})$.
	We now provide a bound on the covering number for the second term. It is easy to verify that the following identity holds
	$$(X^\tr X+\lambda_2 I_d)^{-1}-(X^\tr X+\lambda_1 I_d)^{-1} = (\lambda_1-\lambda_2) (X^\tr X+\lambda_1 I_d)^{-1} (X^\tr X+\lambda_2 I_d)^{-1}$$
	by right- and left-multiplying both sides by $(X^\tr X+\lambda_2 I_d)$ and $(X^\tr X+\lambda_1 I_d)$, respectively.
	Let $\lambda_1,\lambda_2>0$. Then, assuming that $\|x_t\|_2\leq 1$ and $y_t \in [-1,1]$ for all $t$,
	\begin{align*}
	&\norm{ (X_t X+\lambda_2 I_d)^{-1} X^\tr Y-(X^\tr X+\lambda_1 I_d)^{-1} X^\tr Y }_2 \\
	&=  |\lambda_2-\lambda_1|  \norm{ (X^\tr X+\lambda_1 I_d)^{-1} (X^\tr X+\lambda_2 I_d)^{-1} X^\tr Y }_2 \leq  |\lambda_2-\lambda_1|  \frac{1}{\lambda_1\lambda_2} \norm{ X^\tr Y }_2 \leq \left|\lambda_1^{-1} - \lambda_2^{-1}\right| t
	\end{align*}
	Hence, for $\left|\lambda_1^{-1} - \lambda_2^{-1}\right| \leq \delta/T$, we have $\norm{ (X^\tr X+\lambda_2 I_d)^{-1} X^\tr Y-(X^\tr X+\lambda_1 I_d)^{-1} X^\tr Y }_2 \leq \delta$, and thus the discretization of $\lambda^{-1}$ on $(0, \lambda_{\min}^{-1}]$ gives an $\ell_{\infty}$-cover, and the size of the cover at scale $\delta$ is $\lambda_{\min}^{-1}T\delta^{-1}$. The Dudley entropy integral yields the bound of,
	$$\Rad(\Pi, (\x,\y)) \leq %\inf_{\alpha > 0} \left\{4 \alpha+ \frac{12}{\sqrt{T}} \int_{\alpha}^1 \sqrt{\log \mathcal{N}_2(\delta,T)} d\delta \right\} 
	12\sqrt{T} \int_0^1 \sqrt{\log (2T\lambda_{\min}^{-1} \delta^{-1})} d\delta  \leq 12\sqrt{T} \left(1+\sqrt{\log (2T\lambda_{\min}^{-1})}\right).$$
	This concludes the proof.
\end{proof}

\begin{proof}[\textbf{Proof of Lemma \ref{lem:valftrl}}]
Using Theorem \ref{thm:compete_with_strategies},
{\small\begin{align*}
& \Val_T(\Pi_\Lambda) \leq 	2\Rad(\loss, \Pi_\Lambda) = 2\ \sup_{\z, \z'} \mathbb{E}_{\epsilon} \sup_{\w_0 : \norm{\w_0} \le 1, \lambda \in \Lambda}\left[ \sum_{t=1}^T \epsilon_t \ip{\w_0 - \frac{\sum_{i=1}^{t-1} \z_i(\epsilon)}{\max\left\{\lambda\left(\norm{\sum_{i=1}^{t-1} \z_i(\epsilon)} , t\right),  \norm{\sum_{i=1}^{t-1} \z_i(\epsilon)}\right\}} }{\z'_t(\epsilon)} \right]
\end{align*}
}
which we can upper bound by splitting the supremum into two:
{\small
\begin{align*}
2\ \sup_{\z'} \mathbb{E}_{\epsilon} \sup_{\w_0 : \norm{\w_0} \le 1}\left[ \sum_{t=1}^T \epsilon_t \ip{\w_0}{\z'_t(\epsilon)} \right] + 2\ \sup_{\z, \z'} \mathbb{E}_{\epsilon} \sup_{\lambda \in \Lambda}\left[ \sum_{t=1}^T \epsilon_t \ip{\frac{\sum_{i=1}^{t-1} \z_i(\epsilon)}{\max\left\{\lambda\left(\norm{\sum_{i=1}^{t-1} \z_i(\epsilon)} , t\right),  \norm{\sum_{i=1}^{t-1} \z_i(\epsilon)}\right\}} }{\z'_t(\epsilon)} \right] 
\end{align*}}
%& =  2\ \sup_{\z'} \mathbb{E}_{\epsilon}\left\| \sum_{t=1}^T \epsilon_t \z'_t(\epsilon) \right\| + 2\ \sup_{\z, \z'} \mathbb{E}_{\epsilon} \sup_{\lambda \in \Lambda}\left[ \sum_{t=1}^T \epsilon_t \ip{\frac{\sum_{i=1}^{t-1} \z_i(\epsilon)}{\max\left\{\lambda\left(\norm{\sum_{i=1}^{t-1} \z_i(\epsilon)} , t\right),  \norm{\sum_{i=1}^{t-1} \z_i(\epsilon)}\right\}} }{\z'_t(\epsilon)} \right] \\
The first term is simply
\begin{align*}
	2\ \sup_{\z'} \mathbb{E}_{\epsilon}\left\| \sum_{t=1}^T \epsilon_t \z'_t(\epsilon) \right\| \leq 2\sqrt{T}.
\end{align*}
The second term can be written as
\begin{align*}
&2\ \sup_{\z, \z'} \mathbb{E}_{\epsilon} \sup_{\lambda \in \Lambda}\left[ \sum_{t=1}^T \epsilon_t \ip{\frac{\sum_{i=1}^{t-1} \z_i(\epsilon)}{\norm{\sum_{i=1}^{t-1} \z_i(\epsilon)}} }{\z'_t(\epsilon)} \frac{\norm{\sum_{i=1}^{t-1} \z_i(\epsilon)}}{\max\left\{\lambda\left(\norm{\sum_{i=1}^{t-1} \z_i(\epsilon)} , t\right),  \norm{\sum_{i=1}^{t-1} \z_i(\epsilon)}\right\}} \right] \\
& \le  2\ \sup_{\z} \sup_{\s} \mathbb{E}_{\epsilon} \sup_{\lambda \in \Lambda}\left[ \sum_{t=1}^T \epsilon_t \s_t(\epsilon) \frac{\norm{\sum_{i=1}^{t-1} \z_i(\epsilon)}}{\max\left\{\lambda\left(\norm{\sum_{i=1}^{t-1} \z_i(\epsilon)} , t\right),  \norm{\sum_{i=1}^{t-1} \z_i(\epsilon)}\right\}} \right] 
\end{align*}
and the tree $\s$ can be erased (see end of the proof of Theorem~\ref{thm:main_lip_contraction}), yielding an upper bound
\begin{align*}
&2\ \sup_{\z} \mathbb{E}_{\epsilon} \sup_{\lambda \in \Lambda}\left[ \sum_{t=1}^T  \frac{ \epsilon_t \norm{\sum_{i=1}^{t-1} \z_i(\epsilon)}}{\max\left\{\lambda\left(\norm{\sum_{i=1}^{t-1} \z_i(\epsilon)} , t\right),  \norm{\sum_{i=1}^{t-1} \z_i(\epsilon)}\right\}} \right] \\ 
&\le   2\ \sup_{\mbf{a}} \mathbb{E}_{\epsilon} \sup_{\lambda \in \Lambda}\left[ \sum_{t=1}^T  \frac{ \epsilon_t \mbf{a}_t(\epsilon)}{\max\left\{\lambda\left(\mbf{a}_t(\epsilon) , t\right), \mbf{a}_t(\epsilon) \right\}} \right] \\
& \le   2\ \sup_{\mbf{a}} \mathbb{E}_{\epsilon} \sup_{\lambda \in \Lambda}\left[ \sum_{t=1}^T  \frac{ \epsilon_t }{\max\left\{\frac{\lambda\left(\mbf{a}_t(\epsilon) , t\right)}{\mbf{a}_t(\epsilon)}, 1 \right\}} \right] \\
& =   2\ \sup_{\mbf{a}} \mathbb{E}_{\epsilon} \sup_{\lambda \in \Lambda}\left[ \sum_{t=1}^T \epsilon_t \min\left\{ \frac{\mbf{a}_t(\epsilon)}{\lambda\left(\mbf{a}_t(\epsilon) , t\right)}, 1  \right\} \right] \\
& =   2\ \sup_{\mbf{b}} \mathbb{E}_{\epsilon} \sup_{\gamma \in \Gamma}\left[ \sum_{t=1}^T   \epsilon_t \gamma\left(\mbf{b}_t(\epsilon), 1/t\right) \right] \\
& \le  2\ \mathcal{R}_T(\Gamma)
\end{align*}
where in the above $\mbf{a}$ is a $\reals^+$-valued tree such that $\mbf{a}_t : \{\pm 1\}^{t-1} \mapsto [0,t-1]$, $\mbf{b}$ is a $[1/T,1]$-value tree and $\Gamma = \left\{\gamma : \forall b \in [1/T,1], a \in [0,1],  \gamma(a,b) = \min\left\{ \frac{a/(b-1)}{\lambda(a/(b-1),1/b)} , 1\right\}, \lambda \in \Lambda\right\}$. 
\end{proof}

\end{document}